\newcommand {\ent} {\mathrel{{\scriptstyle\mid\!\sim}}}

\newcommand {\esiste} {\exists}

\newcommand {\sx} {\langle}
\newcommand {\dx} {\rangle}

\newcommand {\emme} {\mathcal{M}}

\newcommand {\tc} {\mid}
\newcommand {\vuoto} {\emptyset}

\newcommand{\tip}{{\bf T}}

\newcommand{\alc}{\mathcal{ALC}}
\newcommand{\alct}{\mathcal{ALC}+\tip}

\newcommand{\alctr}{\mathcal{ALC}+\tip_{R}}
\newcommand{\el}{\mathcal{EL}}
\newcommand{\elbot}{\mathcal{EL}^{\bot}}

\newcommand{\eltm}{\mathcal{EL}^{\bot} \tip_{min}}

\newcommand{\sroel}{\mathcal{SROEL}(\sqcap,\times)}
\newcommand{\sroelrt}{\mathcal{SROEL}(\sqcap,\times)^{\Ra}\tip}
\newcommand {\lingconc} {\mathcal{S}}

\newcommand{\wARC}{w^{\footnotesize{A \sqsubseteq \exists R.C}}}

\newcommand{\auxARC}{{aux^{\footnotesize{A \sqsubseteq \exists R.C}}}}

\newcommand{\prj}{{\iota}}

\newcommand{\be}{\begin{enumerate}}
\newcommand{\ee}{\end{enumerate}}

\newcommand{\hide}[1]{}

\def \cases{\left \{\begin{array}{l}}
\def \endcases{\end{array}\right .}

\newcommand {\Ra} {{\bf R}}

\newcommand {\Ri} {\Rightarrow}

\newcommand {\bes} {\begin{description}}
\newcommand{\ens} {\end{description}}

\newcommand {\beq} {\begin{quote}}
\newcommand {\enq} {\end{quote}}
\newcommand {\bit} {\begin{itemize}}
\newcommand {\enit} {\end{itemize}}

\documentclass[runningheads]{llncs}

\usepackage{times}
\usepackage{undertilde}
\usepackage{graphicx}
\usepackage{latexsym}
\usepackage{graphicx}

\usepackage{color}
\usepackage{amssymb}

\def \Ri{\Rightarrow}

\newcommand{\dlltm}{\mbox{\em DL-Lite}_{\mathit{c}}\tip_{min}}

\begin{document}
\bibliographystyle{plain}

\title{Defeasible Reasoning in ${\cal SROEL}$: 
from Rational Entailment to Rational Closure}


\author{
Laura Giordano \and Daniele Theseider Dupr\'e }





\institute{DISIT - Universit\`a del  Piemonte Orientale, Alessandria, Italy - \email{laura.giordano@uniupo.it, dtd@di.unipmn.it}
}

\maketitle

\begin{abstract}
In this work we study a rational extension $\sroel^{\Ra \ }\tip$ of the
low complexity description logic $\sroel$, which underlies the OWL EL ontology language.
The extension involves a typicality operator $\tip$, whose semantics is based on Lehmann and Magidor's ranked models
and allows for the definition of defeasible inclusions. 
We consider both rational entailment and minimal entailment. 
We  show that deciding instance checking under minimal entailment is in general $\Pi^P_2$-hard,
while, under rational entailment, instance checking can be computed in polynomial time.
We develop a Datalog calculus for instance checking under rational entailment
and exploit it, with stratified negation, for computing the rational closure of simple KBs in polynomial time.
\normalcolor
\end{abstract}

\section{Introduction}

The need for extending Description Logics (DLs) with nonmonotonic features has led, in the last decade,
to the development of several extensions of DLs, obtained by combining them with the most well-known formalisms for nonmonotonic reasoning
\cite{BaaderH92,Straccia93,baader95b,donini2002,lpar2007,Eiter2008,kesattler,sudafricaniKR,bonattilutz,casinistraccia2010,rosatiacm,Bonatti2011,KnorrECAI12,CasiniDL2013,AIJ,bonattiAIJ15,AIJ15}
to deal with defeasible reasoning and inheritance,
to allow for prototypical properties of concepts and to combine
DLs with nonmonotonic rule-based languages 
under the answer set semantics
\cite{Eiter2008}, the well-founded semantics \cite{Eiter2011}, the MKNF semantics \cite{rosatiacm,KnorrECAI12}, as well as in Datalog +/- \cite{Gottlob14}.
Systems integrating Answer Set Programming (ASP) \cite{GelfondLeone02,Gelfond} and  DLs have been developed \cite{Xiao2012}. 

In this paper we study a rational extension of the logic $\sroel$, introduced by Kr\"{o}tzsch \cite{KrotzschJelia2010}.
It is a  low-complexity DL of the $\el$ family \cite{rifel}
that includes local reflexivity, role conjunction and concept products and is at the basis of OWL 2 EL.
The rational extension of $\sroel$ is based on  Kraus, Lehmann and Magidor (KLM) preferential semantics \cite{KrausLehmannMagidor:90},
and, specifically, on ranked models \cite{whatdoes}. We call the logic $\sroelrt$ and we define notions of rational and minimal entailment for it. Also, we develop a Datalog calculus for instance checking and subsumption under rational entailment and exploit it to construct the rational closure of a knowledge base using {\em stratified} negation. 

The semantics of ranked interpretations for DLs was first studied  in \cite{sudafricaniKR}, where a rational extension of $\alc$ is developed
allowing for defeasible concept inclusions of the form $C \utilde{\sqsubset} D$.
In this work, following  \cite{FI09,AIJ15}, 
we extend the language of $\sroel$ with typicality concepts of the form $\tip(C)$,
whose instances are intended to be the typical $C$ elements.
Typicality concepts can be used to express defeasible inclusions of the form $\tip(C) \sqsubseteq D$
(``the typical $C$ elements are $D$'').
Here, however, as in \cite{Booth15,ISMIS2015}, 
we allow for typicality concepts to freely occur in concept inclusions. 
In this respect, the language with typicality that we consider is more general than the language in
 \cite{AIJ15}, where $\mathit{\tip}$ may only occur on the left hand side of
inclusions 
as well as in assertions.
For this language, we define a Datalog translation for $\sroelrt$ which builds on the materialization calculus
in \cite{KrotzschJelia2010}, and, for reasoning about typicality, exploits the properties of ranked models, 
by suitably encoding the typicality operator and its properties.
We show that instance checking for $\sroelrt$ can be computed in polynomial time under the rational entailment.
Following the approach developed in \cite{KrotzschJelia2010} for $\sroel$, the materialization calculus is used to define a Datalog calculus for subsumption in  $\sroelrt$.

The rational closure of a knowledge base has been introduced by  Lehmann and Magidor \cite{whatdoes} to allow for stronger inferences
with respect to preferential and rational entailment, \normalcolor
 and several constructions of rational closure have been proposed for $\alc$ \cite{casinistraccia2010,CasiniJAIR2013,CasiniDL2013,AIJ15,TesiMoodley2016}. 
Such constructions are defined for knowledge bases containing strict or defeasible inclusions
and, for the construction in  \cite{AIJ15}, which allows for defeasible inclusions of the form $\tip(C) \sqsubseteq D$
(where $C$ and $D$ are $\alc$ concepts), minimal canonical ranked models have been shown to provide a semantic characterization of rational closure
for $\alc$, thus generalizing to DLs the canonical model result in \cite{whatdoes}.
Based on the same construction introduced in \cite{AIJ15} for $\alc$,
in this paper we use the Datalog calculus for $\sroelrt$ plus stratified negation to compute
in polynomial time the Rational Closure 
for {\em simple} $\sroelrt$ KBs, where the typicality operator only occurs on the left hand side of inclusions.

However, not for all simple knowledge bases the rational closure is consistent. 
Indeed, the minimal canonical model semantics does not provide a general semantic characterization of the rational closure 
for $\sroel$ with typicality, as a KB may have alternative minimal canonical models with incompatible rankings,
or no canonical model at all.
In particular, we show that instance checking in $\sroelrt$ under the minimal canonical model semantics is \textsc{$\Pi^p_2$}-hard.
We observe that, even in cases when the KB has no minimal canonical model, the rational closure, 
when consistent, may still provide meaningful consequences.

A preliminary version of some results in the paper
appeared in  \cite{DL2016}, and a version including the calculus for rational entailment in \cite{CILC2016}.
In this paper we also exploit the calculus for computing the rational closure of a (simple) $\sroelrt$ KB in polynomial time.
Furthermore, we provide a \textsc{$\Pi^P_2$} lower bound on the complexity of instance checking under the minimal canonical model semantics, thus strengthening the result in \cite{DL2016,CILC2016}.
\normalcolor

\section{A rational extension of $\sroel$}
\label{sec-rat}

In this section we recall the logic $\sroelrt$ introduced in \cite{TPLP2016,CILC2016}
extending the notion of concept in $\sroel$ 
by adding typicality concepts. We refer to \cite{KrotzschJelia2010} for a detailed description of the syntax and semantics of $\sroel$.
%
%
%
%
We let ${N_C}$ be a set of concept names, ${N_R}$ a set of role names
  and ${N_I}$ a set of individual names.  
A concept in $\sroel$ is defined as follows:
\begin{quote}
 $C:= A \tc \top \tc \bot \tc  C \sqcap C \tc \exists R.C \tc \exists S.Self  \tc \{a\}$
 \end{quote}
where $A \in N_C$ and $R,S \in {N_R}$.
We introduce a notion of {\em extended concept} $C_E$ as follows:  
\begin{quote}
$C_E:= C \tc \tip(C) \tc C_E\sqcap C_E \tc \exists S.C_E$
 \end{quote}
\noindent
where $C$ is a $\sroel$ concept.
Hence, any concept of $\sroel$ is also an extended concept; a typicality concept $\tip(C)$ is an extended concept and can occur
in conjunctions and existential restrictions, but it cannot be nested.

A KB is a triple $\mathit{(TBox, RBox, ABox)}$. $\mathit{TBox}$ contains a finite set
of  {\em general concept inclusions} (GCI) $C \sqsubseteq D$, where $C$ and $D$ are extended concepts;
$\mathit{RBox}$ (as in \cite{KrotzschJelia2010}) contains a finite set of {\em role inclusions} of the form $S \sqsubseteq T$,
{\em generalized role inclusions} of the form
$R \circ S \sqsubseteq T$,
{\em role conjunction axioms}
$R \sqcap S \sqsubseteq T$ and
{\em concept product axioms}  $C \times D \sqsubseteq T$ and
$R \sqsubseteq C \times D$, where $C$ and $D$ are extended concepts, and $R,S,T$ are role names  in $N_R$.
$\mathit{ABox}$ contains {\em individual assertions} of the form $C(a)$ and $R(a,b)$, where $a, b \in N_I$, $R \in N_R$ and $C$ is an extended concept.
Restrictions are imposed on the use of roles as in \cite{KrotzschJelia2010}
(and, in particular, all the roles occurring in
{\em Self} concepts and in role conjunctions must be {\em simple roles}, roughly speaking,  roles which do not include the composition of other roles).

We define a semantics for  $\sroelrt$  based on ranked models \cite{whatdoes}.
As done in \cite{AIJ15} for $\alc$, we define the semantics of $\sroelrt$ by 
adding to  $\sroel$ interpretations \cite{KrotzschJelia2010} a \emph{preference relation} $<$ on
the domain, which is intended to compare the ``typicality''
of domain elements.
The typical instances of a concept $C$, i.e., the instances of
$\tip(C)$, are the instances of $C$ that are minimal with respect
to $<$. 
The properties of the $<$ relation are defined in agreement with the properties of the preference relation
in Lehmann and Magidor's {\em ranked models} \cite{whatdoes}.
A semantics for DLs with defeasible inclusions based on ranked models was first proposed in \cite{sudafricaniKR}.
\begin{definition}
\label{semalctr}
A $\sroelrt$ interpretation $\emme$ is any
structure $\langle \Delta, <, \cdot^I \rangle$ where:
\begin{itemize}
\item
$ \Delta$ is a domain;
 $\cdot^I$ is an interpretation function that maps each
concept name $A\in N_C$ to a set $A^I \subseteq  \Delta$, each role name $R \in N_R$
to  a binary relation $R^I \subseteq  \Delta \times  \Delta$,
and each individual name $a \in N_I$ to an element $a^I \in  \Delta$;
the extension of $\cdot^I$ to complex concepts is defined as usual:
\begin{quote}
$\top^I=\Delta$; \ \ \ \ \ \ \ $\bot^I=\emptyset$; \ \ \ \ \ \ \
$\{a\}^I= \{a^I\}$; \\
$(C \sqcap D)^I$= $C^I \cap D^I$; \\
$(\esiste R.C)^I$= $\{x \in \Delta \tc \exists y \in C^I: (x,y) \in R^I \}$; \\
$(\exists R.Self)^I$= $\{x \in \Delta \tc (x,x) \in R^I\}$; 
\end{quote}
\item
$<$ is an irreflexive, transitive, 
well-founded and modular
relation over $\Delta$;
\item
the interpretation of concept $\tip(C)$ is defined as:
$(\tip(C))^I = Min_<(C^I)$,
where $Min_<(S)= \{u: u \in S$ and $\nexists z \in S$ s.t. $z < u \}$.
\end{itemize}
\end{definition}
\noindent Furthermore, an irreflexive and transitive relation $<$ is  {\em well-founded} if,
for all  $S \subseteq \Delta$, for all $x \in S$, either $x \in Min_<(S)$
or $\exists y \in  Min_<(S)$ such that $y < x$. It is {\em modular} if, for all $x,y,z \in \Delta$, $x <y$ implies $x<z$ or $z<y$.
The well-foundedness condition guarantees that
if, for a non-extended concept $C$, 
there is a $C$ element in $\emme$, then there is a minimal $C$ element in $\emme$
(i.e., $C^I\neq \emptyset$ implies $(\tip(C))^I\neq \emptyset$). 

In the following, we will refer to $\sroelrt$ interpretations as {\em ranked interpretations}.
Indeed, as in  \cite{whatdoes}, modularity in preferential models can be equivalently defined by postulating the existence of
a rank function $k_{\emme}: \Delta \longmapsto \Omega$, where $\Omega$ is a totally ordered set.
The preference relation $<$ can be defined from $k_{\emme}$ as follows:
$x < y$ if and only if $k_{\emme}(x) < k_{\emme}(y)$.
Hence, in the following, we will assume that a rank function $k_{\emme}$ is always associated with any model $\emme$.
We also define the {\em rank $k_{\emme}(C)$ of a concept $C$ in the model $\emme$} as $k_{\emme}(C) = min\{k_{\emme}(x) \tc x \in C^I\}$ (if $C^I=\vuoto$, then
$C$ has no rank and we write $k_{\emme}(C)=\infty$).

The semantics of the typicality operator defined above is the same as the one in \cite{AIJ15} for
$\alctr$.
Similarly to other concept constructors, the typicality operator can be used in TBox and ABox
with different restrictions, depending on the description logic.
In \cite{AIJ15}, $\tip(C)$ can only occur on the left-hand side of concept inclusions
(namely, in typicality inclusions of the form
$\tip(C) \sqsubseteq D$).
Here, we call  {\em simple} KBs the ones which respect this restriction, but, 
as in \cite{Booth15,ISMIS2015},  we also consider the general case with no restrictions on the occurrences
of typicality concepts $\tip(C)$ in concept inclusions and assertions.
Instead, as in $\sroel$, we do not allow negation, union and universal restriction which are allowed in $\alc$.
Given 
an interpretation $\emme$, 
the notions of satisfiability and entailment are defined as usual:

\begin{definition}[Satisfiability and rational entailment]\label{Def-ModelSatTBox-ABox}
An interpretation $\emme=\sx\Delta, <, \cdot^I\dx$ satisfies:\\
$\bullet$ \  a concept inclusion $C \sqsubseteq D$ if \  $C^I \subseteq D^I$;\\
$\bullet$ \   a {role inclusion} $S \sqsubseteq T$ if  \ $S^I \subseteq T^I$;\\
$\bullet$ \  a {generalized role inclusion} $R \circ S \sqsubseteq T$ if \  $R^I \circ S^I \subseteq T^I$\\
$\bullet$ \   a {role conjunction axiom} $R \sqcap S \sqsubseteq T$ if \  $R^I\cap S^I \subseteq T^I$;\\
$\bullet$ \   a {concept product axiom} $C \times D \sqsubseteq T$  if \ $C^I \times D^I \subseteq T^I$;\\
$\bullet$ \  a {concept product axiom} $R \sqsubseteq C \times D$ if \ $R^I \subseteq C^I \times D^I$;\\
$\bullet$ \   an assertion $C(a)$ if $a^I \in C^I$;\\
$\bullet$ \   an assertion $R(a,b)$ if $(a^I,b^I) \in R^I$.

Given  a KB $\mathit{K=(TBox, RBox, ABox)}$, an interpretation $\emme=$$\sx \Delta, <, \cdot^I \dx$ {\em satisfies}
$\mathit{TBox}$  (resp., $\mathit{RBox}$, $\mathit{ABox}$) if
$\emme$ satisfies   all  axioms in $\mathit{TBox}$ (resp., $\mathit{RBox}$, $\mathit{ABox}$),
and we write $\emme \models \mathit{TBox}$ (resp., $\mathit{RBox}$, $\mathit{ABox}$).
An interpretation $\emme = \sx \Delta, <, \cdot^I \dx$ is a {\em model} of $K$ (and we write $\emme \models K$) if
$\emme$ satisfies   all the axioms in $\mathit{TBox}$, $\mathit{RBox}$ and $\mathit{ABox}$.

Let a query $F$ be either a concept inclusion $C\sqsubseteq D$ (where $C$ and $D$ are extended concepts) or an individual assertion.
{\em $F$ is rationally entailed by $K$}, written $K \models_{sroelrt} F$, if for all models $\emme=$$\sx \Delta, <, \cdot^I \dx$ of $K$,
$\emme$ satisfies $F$.
In particular, the {\em instance checking} problem (under rational entailment) is the problem of deciding whether an assertion ($C(a)$, $\tip(C)(a)$ or $R(a,b)$) is rationally entailed by $K$.


\end{definition}

Given the correspondence of typicality inclusions with conditional assertions $C \ent D$,
it can be easily seen
that each ranked interpretation $\emme$
satisfies the following semantic conditions, which are related to Lehmann and Magidor's postulates of
rational consequence relation \cite{whatdoes} reformulated in terms of typicality:

\vspace{-0.3cm}
\begin{tabbing}
$\mathit{(CM)}$ \= xxxxxxxxxx \= \kill \\
$\mathit{(S\mbox{-}LLE)}  ~ \mbox{ If } \emme \models \mathit{A \equiv B}  \mbox{ and }  \emme \models \mathit{\tip(A) \sqsubseteq C} \mbox{ then }   \emme \models \mathit{ \tip(B) \sqsubseteq C} $ \\
$\mathit{(S\mbox{-}RW)}  ~ \mbox{ If } \emme \models  \mathit{B \sqsubseteq  C}  \mbox{ and }      \emme \models \mathit{\tip(A) \sqsubseteq B}    \mbox{ then }   \emme \models \mathit{ \tip(A) \sqsubseteq C} $ \\
$\mathit{(S\mbox{-}Refl)} ~  \emme \models \mathit{\tip(A) \sqsubseteq A } $ \\
$\mathit{(S\mbox{-}And)}  ~ \mbox{ If }  \emme \models \mathit{\tip(A) \sqsubseteq B}  \mbox{ and }   \emme \models \mathit{\tip(A) \sqsubseteq C} \mbox{ then }   \emme \models \mathit{\tip(A) \sqsubseteq  B \sqcap C} $ \\
$\mathit{(S\mbox{-}Or)}  ~ \mbox{ If }  \emme \models \mathit{\tip(A)  \sqsubseteq C}  \mbox{ and }  \emme \models  \mathit{\tip(B) \sqsubseteq C} \mbox{ then }   \emme \models \mathit{\tip(A \sqcup B) \sqsubseteq C} $ \\
$\mathit{(S\mbox{-}CM)}  ~ \mbox{ If }  \emme \models \mathit{\tip(A) \sqsubseteq B}  \mbox{ and }   \emme \models \mathit{\tip(A) \sqsubseteq C} \mbox{ then }  \mathit{\tip(A \sqcap B) \sqsubseteq C} $ \\
$\mathit{(S\mbox{-}RM)}  ~ \mbox{ If }  \emme \models \mathit{\tip(A) \sqsubseteq C}  \mbox{ and }   \emme \not \models \mathit{\tip(A)  \sqsubseteq \neg B} \mbox{ then }  \emme \models \mathit{\tip(A \sqcap B) \sqsubseteq C} $ \\
\end{tabbing}
\vspace{-0.3cm}
It is easy to show that these semantic properties (where the interpretation of $A \sqcup B$, which is not in the language, is the usual one) hold in all the ranked models.
A similar formulation of the semantic properties in terms of defeasible inclusions has been previously provided by Britz et al.\ in \cite{sudafricaniKR},
for the ranked semantics and a proof can be found in \cite{TesiMoodley2016}. Another reformulation of the properties in terms of a selection function semantics was given in \cite{FI09} for the preferential semantics and in \cite{AIJ15} for the rational semantics.
In particular, observe that property $\mathit{(S\mbox{-}RM)}$ can be reformulated as:

if $ \mathit{ (\tip(A) \sqcap B)^I \neq \emptyset}$, then $\mathit{(\tip(A \sqcap B))^I \subseteq (\tip(A))^I}$

\noindent
and, in this form, it is a rephrasing of property $(f_\tip - R)$, in the semantics  with selection function of the operator $\tip$ studied in \cite{AIJ15} 
(Appendix A) for $\alctr$.
This property has a syntactic counterpart in $\sroel$ in the axiom $\exists U. (\tip(A) \sqcap B) \sqcap \tip(A \sqcap B) \sqsubseteq  \tip(A)$,
where $U$ is the universal role ($\top \times \top \sqsubseteq U$),
which holds in all the ranked models.
Observe that this axiom, as well as the property $\mathit{(S\mbox{-}RM)}$, is weaker than
the Lehmann and Magidor's postulate $\mathit{(RM)}$ in \cite{whatdoes}, which would rather be reformulated, for a knowledge base $K$, as:

$\mathit{(RM)}  ~ \mbox{ If }  K \models \mathit{\tip(A) \sqsubseteq C}  \mbox{ and }   K \not \models \mathit{\tip(A)  \sqsubseteq \neg B} \mbox{ then }  K \models \mathit{\tip(A \sqcap B) \sqsubseteq C} $

\noindent
and does not hold in $\sroelrt$ (while it will hold under minimal entailment).

Consider the following example of knowledge base, stating that: typical Italians have black hair; typical students are young;
they hate math, unless they are nerd (in which case they love math); all Mary's friends are typical students.
We also have the assertions stating that  Mary is a student, that Mario is an Italian student, and is a friend of Mary,
Luigi is a typical Italian student, Paul is a typical young student, and Tom is a typical nerd student.
\normalcolor
\begin{example}\label{exa:1}

$\mathit{TBox}$:

$(a) ~ ~ \tip(\mathit{Italian}) \sqsubseteq \mathit{ \exists hasHair.\{Black\}}$

$(b) ~ ~ \tip(\mathit{Student})\sqsubseteq\; \mathit{Young}$ \ \ \ \ \ \ \ \ \ \ \ \ \ \ \ \ \ \ \ \ \ \ \

$(c) ~ ~ \tip(\mathit{Student})\sqsubseteq\; \mathit{MathHater}$

$(d) ~ ~ \tip(\mathit{Student} \sqcap \mathit{Nerd})\sqsubseteq\; \mathit{MathLover}$

$(e) ~ ~ \mathit{\exists hasHair.\{Black\}} \sqcap \mathit{\exists hasHair.\{Blond\}} \sqsubseteq \bot$

$(f) ~ ~ \mathit{MathLover} \sqcap \mathit{MathHater} \sqsubseteq \bot$
\ \ \ \ \ \ \ \ \ \ \

$(g) ~ ~ \exists \mathit{friendOf}.\{\mathit{mary}\} \sqsubseteq \tip(\mathit{Student})$


\smallskip
\noindent
$\mathit{ABox}$: 
$ \mathit{Student}(\mathit{mary}), ~
\mathit{friendOf}(\mathit{mario}, \mathit{mary}), ~ $
$(\mathit{Student} \sqcap \mathit{Italian})(\mathit{mario}), $\\
$ \tip (\mathit{Student} \sqcap \mathit{Italian})(luigi), ~ $
$\tip (\mathit{Student} \sqcap \mathit{Young})(paul),  $
$\tip (\mathit{Student} \sqcap \mathit{Nerd})(tom)  $ 

\end{example}

The fact that concepts $\tip(C)$ can occur anywhere (apart from being nested in a $\tip$ operator)
can be used, e.g., to state that typical working students inherit properties of typical students
($\tip(\mathit{Student} \sqcap \mathit{Worker}) \sqsubseteq \tip(\mathit{Student})$),
in a situation in which typical students and typical workers
have conflicting properties (e.g., as regards paying taxes).
Also, we could state that there are typical students who are Italian:
$\top \sqsubseteq \exists U. \tip(\mathit{Student} \sqcap \mathit{Italian})$.

Standard DL inferences hold for $\tip(C)$ concepts and $\tip(C) \sqsubseteq D$ inclusions.
For instance, we can conclude that Mario is a typical student (by (g)) and young (by (b)).
However, by the properties of defeasible inclusions, Luigi, who is a typical Italian student, and Paul, who is a typical young student,
both inherit the property of typical students of being math haters, respectively, by properties (S-RM) and (S-CM).
Instead, as Tom is a typical nerd student, and typical nerd students are math lovers, this specific property of typical nerd students
prevails over the less specific property of typical students of hating math.
So we can consistently conclude that Tom is a $\mathit{MathLover}$.

A normal form for $\sroelrt$ knowledge bases can be defined.
%
A {KB in $\sroelrt$ is in {\em normal form} if it admits all the axioms of a $\sroel$ KB in normal form:
\begin{center}
$C(a)$ \ \ \ \ \ \ \ $R(a,b)$  \ \ \ \ \ \ \ $A \sqsubseteq \bot$ \ \ \ \ \ \ \ $\top \sqsubseteq C$ \ \ \ \ \ \ \ $A \sqsubseteq \{c\}$  \\
$A \sqsubseteq C$ \ \ \ \ \ $A \sqcap B \sqsubseteq C$ \ \ \ \ \ $\exists R. A \sqsubseteq C$ \ \ \ \ \ $A \sqsubseteq \exists R. B$ \\
$\{a\} \sqsubseteq C$ \ \ \ \ \ $\exists R. \mathit{Self} \sqsubseteq C$ \ \ \ \ \ $A \sqsubseteq \exists R. \mathit{Self}$\\
$R \sqsubseteq T$  \ \ \ \ \ \ \ $R \circ S \sqsubseteq T$  \ \ \ \ \ \ \ $R \sqcap S \sqsubseteq T$  \ \ \ \ \ \ \ $A \times B \sqsubseteq R$  \ \ \ \ \ \ \ $R \sqsubseteq  C \times D $
\end{center}
(where $A,B,C,D \in N_C$, $R,S,T \in N_R$ and $a,b,c \in N_I$)
and, in addition, it admits axioms of the form:
$A\sqsubseteq T(B)$\ \  and \ \ $T(B) \sqsubseteq C$
with $A,B,C \in N_C$.
Extending the results in \cite{rifel} and in \cite{KrotzschJelia2010}, it is easy to see that, given a $\sroelrt$ KB, a semantically equivalent KB in normal form (over an extended signature) can be computed in linear time.
In essence, for each concept $\tip(C)$ occurring in the KB, we introduce two new concept names, $X_C$ and $Y_C$.
A new KB is obtained by replacing all the occurrences of  $\tip(C)$  with $X_C$
in all the inclusions and assertions, and adding the following additional inclusion axioms:
\begin{center}
$X_C \sqsubseteq \tip(Y_C)$, \ \ $\tip(Y_C) \sqsubseteq X_C$, \ \ $Y_C \sqsubseteq C$, \ \  $C \sqsubseteq  Y_C$
\end{center}
Then the new KB undergoes the normal form transformation for $\sroel$ \cite{KrotzschJelia2010}.
The resulting KB is linear in the size of the original one.

\begin{example}
Considering the TBox in Example \ref{exa:1},
inclusion 

$(a) ~ ~ \tip(\mathit{Italian}) \sqsubseteq \mathit{ \exists hasHair.\{Black\}}$\\
is transformed in the following set of inclusions:

$(a_1) ~ ~ X_I \sqsubseteq \mathit{ \exists hasHair.\{Black\}}$ \ \ \ \ \  
$(a_2) ~ ~ X_I  \sqsubseteq \tip(\mathit{Italian})$ \ \ \ \ \ 
$(a_3) ~ ~   \tip(\mathit{Italian}) \sqsubseteq X_I $



\noindent
Inclusion 
$(d) ~ ~ \tip(\mathit{Student} \sqcap \mathit{Nerd})\sqsubseteq\; \mathit{MathLover}$
is mapped to the set of inclusions:

$(d_1) ~ ~ X_{SN}  \sqsubseteq\; \mathit{MathLover}$ \ \ \ \ \ 
$(d_2) ~ ~ X_{SN}  \sqsubseteq\; \tip(\mathit{Y_{SN}})  $ \ \ \ \ \ 
$(d_3) ~ ~\tip(\mathit{Y_{SN}}) \sqsubseteq\; X_{SN}$

$(d_4) ~ ~ \mathit{Student} \sqcap \mathit{Nerd} \sqsubseteq\; Y_{SN}$ \ \ \ \ \ 
$(d_5) ~ ~ Y_{SN}  \sqsubseteq\;  \mathit{Student} \sqcap \mathit{Nerd}  $

\noindent
Then $(a_1)$ is transformed further (the normal form transformation for $\sroel$)  into:
\ \ 
$(a_1') ~  X_I \sqsubseteq \mathit{ \exists hasHair.B}$, and 
$(a_1'') ~ {B \sqsubseteq  \{Black\}}$, while $(d_5)$ is split in two inclusions.

All the other axioms in the TBox, apart from (b) and (c), have to be transformed in normal form.
Assertions are also subject to the normal form transformation.
For instance, $\tip (\mathit{Student} \sqcap \mathit{Nerd})(tom)  $
becomes $X_{SN}(tom)$, where $X_{SN}$ is one of the concept names introduced above.

\end{example}

\section{Minimal entailment} \label{sec:minimal entailment}

In Example \ref{exa:1}, we cannot conclude that all typical young Italians have black hair (and in case Bob is a typical young Italian, he has black hair) using property  (S-RM) above, 
as we do not know whether there is some typical Italian who is also young.
To support such a stronger nonmonotonic inference, a minimal model semantics is needed to select those interpretations where individuals are as typical as possible.
Among models of a KB, we select the minimal ones according to the following {\em preference relation $\prec$ over the set of ranked interpretations}.
An interpretation $\emme = $$\langle \Delta, <, I \rangle$ {\em is preferred to} $\emme' =
\langle \Delta', <', I' \rangle$
($\emme \prec \emme'$) if:
$\Delta = \Delta'$;
 $C^I = C^{I'}$ for all non-extended concepts $C$;
for all $x \in \Delta$, 
$ k_{\emme}(x) \leq k_{\emme'}(x)$, and there exists
$y \in \Delta$ such that $ k_{\emme}(y) < k_{\emme'}(y)$.
We say that an interpretation $\emme = $$\langle \Delta, <, I \rangle$ {\em is a minimal model of $K$} if there is no model $\emme' = $$\langle \Delta', <,' I' \rangle$ of $K$ such that $\emme' \prec \emme$. 

We can see that in all the minimal models of the KB in Example \ref{exa:1}, $luigi$ is an instance of the concept
 $\mathit{ \exists hasHair.\{Black\}}$ and the inclusion $\mathit{\tip (Young \sqcap Italian) \sqsubseteq}$ $\mathit{ \exists hasHair.\{Black\}}$ is satisfied,
as nothing prevents a $\mathit{Young \sqcap Italian}$ individual from having rank $0$.

In particular, we consider the notion of minimal canonical model defined in \cite{AIJ15} to  capture rational closure of an $\alc$ KB extended with typicality.
The requirement of a model to be canonical is used to guarantee that models contain enough individuals.
Given a KB $K$ and a query $F$,  let $\lingconc$ be the set of all the (non-extended) concepts (and subconcepts) occurring  in $K$ or $F$
together with their complements ($\lingconc$ is finite).
In the following, we assume that all concepts occurring in the query $F$ are included in $K$.
\begin{definition}[Canonical models]\label{canonical-model}
A  model $\emme=\sx \Delta, <, I \dx$ of $K$ is
{\em canonical} if,  for each set of $\sroelrt$ concepts
$\{C_1, C_2, \dots,$ $ C_n\} \subseteq \lingconc$ consistent with $K$ (i.e., s.t. $K \not\models_{sroelrt} C_1 \sqcap C_2 \sqcap \dots \sqcap C_n \sqsubseteq \bot$),
there exists (at least) a domain element $x \in \Delta$ such that
$x \in (C_1 \sqcap C_2 \sqcap \dots \sqcap C_n)^I$.
\end{definition}

\begin{definition}
\label{minimal-model}
$\emme$ is a {\em minimal canonical model} of $K$
if it is a canonical model of $K$ and it is minimal 
with respect to the preference relation  $\prec$.
\end{definition}
\begin{definition}[Minimal entailment]\label{minimal-entailment}
Given  
a query $F$,
{\em $F$ is minimally entailed by $K$}, written $K \models_{c\_min} F$ if, for all minimal canonical models $\emme$ 
of $K$, $\emme$ satisfies $F$.
\end{definition}

We can show that instance checking in $\sroelrt$ under minimal entailment is \textsc{$\Pi^P_2$}-hard.
The proof is based on a reduction of the minimal entailment problem of
{\em positive disjunctive logic programs}, which has been proved to be a
 \textsc{$\Pi^P_2$}-hard problem by Eiter and Gottlob in \cite{Eiter95}.
A similar reduction has been used
to prove  \textsc{$\Pi^P_2$}-hardness of entailment for Circumscribed
Left Local $\mathit{EL}^\bot$ knowledge bases in \cite{Bonatti2011},
and in \cite{TPLP2016} to show that  instance checking under the
$\tip$-minimal model semantics (a different semantics)
is a
\textsc{$\Pi^P_2$}-hard problem for $\sroelrt$ knowledge bases.
The reduction in \cite{TPLP2016} (Appendix B  in the ``Supplementary materials") does not work for the minimal canonical model semantics, 
since the resulting knowledge base may have no canonical model, but a simplification of it does work.  

Let us recall the minimal entailment problem of
{\em positive disjunctive logic programs} \cite{Eiter95}.
Let $PV = \{p_1, \ldots, p_n \}$ be a set of propositional variables.
A clause is a formula $l_1 \vee \ldots \vee l_h$, where each literal $l_j$  is either a
propositional variable $p_i$ or its negation $\neg p_i$.
A positive disjunctive logic program (PDLP) is a set of clauses $S= \{\gamma_1, \ldots, \gamma_m\}$,
where each $\gamma_j$ contains at least one positive literal.
A truth valuation for $S$ is a set $I \subseteq PV$,
containing the propositional variables which are true. A truth valuation is a model of $S$
if it satisfies all clauses in $S$.
For a literal $l$, we write $S \models_{min} l$ if and only if every minimal model (with respect to subset inclusion)
of $S$ satisfies $l$. The minimal-entailment problem can be then defined as follows:
given a PDLP $S$ and a literal $l$, determine whether  $S \models_{min} l$.
In the following we sketch the reduction of the minimal-entailment problem for a  PDLP $S$
to the instance checking problem under minimal entailment, from a knowledge base $K$ constructed from $S$.

We define a KB $\mathit{K = (TBox,RBox,ABox)}$ in $\sroelrt$ as follows.
We introduce a concept name $P_h\in N_C$ for each variable $p_h\in PV$ ($h=1,\ldots,n$).
Also, we introduce in $N_C$ an auxiliary concept $H$, a concept name $D_S$ associated with the set of clauses $S$,
and a concept name $D_j$ associated with each clause $\gamma_j$ in $S$ ($j=1,\ldots,m$).
We let $a \in N_I$ be an individual name,
and we define $K$ as follows:
$\mathit{RBox}=\emptyset$,
$\mathit{ABox}=\{ P_h (a), h=1,\ldots,n \} \cup \{\tip(H)(a), D_S(a)\}$,
and $\mathit{TBox}$ contains the following inclusions
(where $ C_i^j$ and $\overline{ C_i^j}$ are concepts associated with each literal $l_i^j$ occurring in $\gamma_j = l_1^j \vee \ldots \vee l_k^j$, as defined below):

(1) $\tip(\top) \sqcap H \sqsubseteq \bot$
\ \ \ \ \ \ \ \ \ \ \ \ \ \ \ \ \ \ \ \ \ \ \ \ \ \ 
(2) $ C_i^j \sqsubseteq D_j$ \ \ \ for all $\gamma_j = l_1^j \vee \ldots \vee l_k^j$ in $S$

(3) $ D_j \sqcap \overline{ C_1^j}  \sqcap \ldots \sqcap \overline{ C_k^j} \sqsubseteq \bot$
\ \ \  for all $\gamma_j = l_1^j \vee \ldots \vee l_k^j$ in $S$

(4) $ D_1  \sqcap \ldots  \sqcap D_m \sqsubseteq D_S$  
\ \ \ \ \ \ \ \ \ \ \ \ \ \ \ \ 
(5) $ D_S \sqsubseteq  D_1  \sqcap \ldots  \sqcap D_m$  

\noindent
for each $h=1,\ldots,n$,
 $j=1,\ldots,m$, 
and where $C_i^j$ and  $\overline{C_i^j}$ (for $i= 1, \ldots, k$) are defined as follows:
\[
   C_i^j=
\begin{cases}
	\tip(P_h) \mbox{\ \ \ \ \ \ \ \ \ \ \ \ \ \ \ \ \ \ \ \ \ \ \ \ \ \  if } l_i^j=p_h\\
    	\exists U.(\tip(\top) \sqcap P_h) \mbox{\ \ \ \ \ \ \ \ \ if } l_i^j=\neg p_h
\end{cases}
\]

\[
   \overline{C_i^j}=
\begin{cases}
	\exists U.(\tip(\top) \sqcap P_h) \mbox{\ \ \ \ \ \ \ \ \ if } l_i^j=p_h\\
    	\tip(P_h) \mbox{\ \ \ \ \ \ \ \ \ \ \ \ \ \ \ \ \ \ \ \ \ \ \ \ \ \  if }  l_i^j=\neg p_h
\end{cases}
\]
where $U$ is the universal role.

Let us consider an arbitrary model $\emme$$=\sx\Delta, <, \cdot^I\dx$ of $K$.
Observe that, by (1), all the $\tip(\top)$ instances are $\neg H$ instances.
Hence, $a^I$ (being a typical $H$) must have rank greater than 0,
and it will have rank 1 in all minimal canonical models.
Inclusions (2) and (3) force the instances of $D_j$ to be the union of the instances of $C_1^j, \ldots, C_k^j$.
Inclusions (4) and (5) force the instances of $D_S$ to be the intersection of the instances of $D_1, \ldots, D_m$.

The minimal canonical models of $K$ satisfying $D_S(a)$ are intended to correspond
to the (propositional) minimal interpretation $J$ satisfying $S$.
Roughly speaking, the concepts $P_h$ such that
$a^I \in (\tip(P_h))^I$ in $\emme$ correspond to the variables $p_h$ true in the interpretation $J$ satisfying $S$.
In any minimal canonical model of $K$, 
either $P_h$ has rank 0 (and $a$ is not a typical $P_h$),
or $P_h$ has rank 1 (and $a$ is a typical $P_h$).
Also, a minimal model of $K$ 
in which the ranking of a set of $P_h$'s is 0, 
is preferred to the models in which the ranking of some of those $P_h$'s is higher (i.e., 1).
This captures the subset inclusion minimality in the interpretations of the positive disjunctive logic program $S$. 
The assertion $D_S(a)$ in $ABox$ is required to select only those interpretations satisfying the set $S$ of disjunctions.
Observe also that $a$ is an istance of $P_h$ for all $h$ (due to the Abox) but it may be a
non-typical instance of $P_h$.

In  any minimal canonical model ${\emme}$ of $K$: 
either $a^I \in (\tip(P_h))^I $ or 
$\mathit{a^I \in  (\exists U.(\tip(\top)}$ $\mathit{ \sqcap \tip(P_h)))^I}$.
Hence, for $a^I$ the two concepts in the definition of $C_i^j$ are disjoint and complementary,
and $\overline{C_i^j}$ is actually the concept representing the complement of $C_i^j$.

\medskip

\begin{proposition}
Given a set $S$ of clauses and a literal $L$,\\
$\; \mbox{\ \ \ \ \ \ \ \ \ \ \ \ \ \ \ \ \  \ \ \ \ \ \ \ \ \ \ \ \ \  \ \ \ \ \ \ \ \ \ \ \ \ \ }$ $S \models_{min}  L$ \ \ \  if and only if \ \ \ \ $K \models_{c\_min} C_L(a)$\\
where $K$ is the KB associated with $S$ as above and $C_L$ is the concept associated with $L$, i.e., $C_L=\tip(P_h)$ if $L=p_h$, and
$C_L=\exists U.(\tip(\top) \sqcap P_h)$ if $L=\neg p_h$.

\end{proposition}
\begin{proof}
($\Ri$) We prove that if $K \not \models_{c\_min} C_L(a)$ then  $S \not \models_{min}  L$.
Let $\emme=\sx\Delta, <, \cdot^I\dx$ be a minimal canonical model of $K$ falsifying  $C_L(a)$.
We want to construct a (propositional) interpretation $J$ satisfying S and falsifying $L$.
Let $J$ be the set of all the variables $p_h$ such that $a$ is a typical $P_h$ element, i.e., 
$J = \{ p_h : \; a^I \in \tip(P_h)^I \}$.
We show that $J$ is a minimal model of $S$ that falsifies $L$. The fact that $J$ is a model of $S$
can be easily shown from the fact that $p_h \in J$ iff  $\emme \models  \tip(P_h)(a)$,
while $p_h \not \in J$ iff  $\emme \not \models  \tip(P_h)(a)$ i.e., iff $\emme  \models  \exists U.(\tip(\top) \sqcap P_h)(a)$.
In fact when $\emme \not \models  \tip(P_h)(a)$, 
$P_h$ must have rank lower than $a$, i.e. rank $0$.
Hence, a literal $l_i^j$ is true in $J$ iff $C_i^j(a)$ is satisfied in $\emme$ and it is false in $J$ if $\overline{C_i^j}(a)$ is satisfied in $\emme$.
From the facts that the concepts $C_i^j$ and $\overline{C_i^j}$ are disjoint and complementary for $a$, that inclusions (2)-(5) are satisfied and that
$D_S(a)$ is satisfied as well, it follows that the interpretation $J$ satisfies all the clauses in $S$. 

Consider a clause $\gamma_j$ in $S$. As  $D_S(a)$ is satisfied in $\emme$, $D_j(a)$ must be satisfied as well, by (5).
By (3) it is not the case that $(\overline{C_1^j} \sqcap \ldots \sqcap \overline{C_k^j})(a)$ is satisfied in $\emme$. 
There must be a $C_h^j$ for $h=1,\ldots,k$, such that $C_h^j(a)$ is satisfied in $\emme$. Hence, the literal $l_h^j$ must be satisfied in $J$.
Thus, $\gamma_j$ is satisfied in $J$. This holds for all the clauses $\gamma_j$ in $S$. Hence, $J$ satisfies $S$.

In a similar way it can be seen that, as  $C_L(a)$ is falsified in $\emme$, then the literal $L$ is falsified in $J$.
The minimality of $J$ can be proved 
by contradiction.
If $J$ were not minimal, there would be a model $J'$ of $S$, with $J' \subseteq J$. 
First observe that, for any valuation $I \subseteq PV$ we can define a concept  $C_I$  obtained as the conjunction of the name concepts in the set $\Gamma_I=\{ P_h : \; p_k \in I \} \cup \{ \neg P_h : \; p_k \not \in I \}$ and $C_{I}$ is consistent with $K$. In fact, we can always add to a model of $K$ a new domain element with rank 1 satisfying $C_I$ as well as the inclusions (1)-(5) (by properly defining the evaluation of the concepts $D_j$'s and of $D_S$), to obtain a new model of $K$. In particular, in a canonical model, for each propositional valuation $I \subseteq PV$, 
there must be a domain element which is an instance of $C_I$.
For $J' \subseteq J$,  let $I=J-J'$. There must be an element $z\in \Delta$ of $\emme$ which is an instance of $C_I$ and, furthermore,
$k_{\emme}(z)\geq 1$ (if not, $\tip(P_h)(a)$ would not be satisfiable in $\emme$ for all $p_h \in I=J-J'$, contradicting the fact that, by construction of $J$, for each $p_h \in J$, $\tip(P_h)(a)$ is satisfiable in $\emme$). 
We can define a model $\emme'$ such that  $\emme' \prec \emme$, by only changing the rank of $z$ in $\emme$ to $0$. 
We can easily see that $\emme'$ is still a model of $K$, a model in which, for all $p_h \in J-J'$, $\tip(P_h)(a)$ is not satisfied, but in which 
inclusions (1)-(5) and the ABox assertions are still satisfied (given that $J'$ is a model of $S$ and that $\emme' \models  \tip(P_h)(a)$ iff $p_h \in J'$). 
This contradicts the hypothesis that $\emme$ is minimal.

\noindent
($\Leftarrow$)
We prove that if $S \not \models_{min}  L$ then $K \not \models_{c\_min} C_L(a)$.

Let $J$ be a (propositional) minimal interpretation satisfying $S$ and falsifying $L$.
We  build a  minimal canonical model  of $K$ falsifying $C_L(a)$, as follows.
Let $\emme$$=\sx\Delta, <, \cdot^I\dx$ be defined as follows (where $r=1,2$ and $h=1, \ldots, n$ ):  

$\Delta= \{u, v\} \cup \{x_V^1, x_V^2: \; V \subseteq PV \mbox{ is a propositional valuation}\}$; \ \ \ $a^I= u$;

$k_{\emme}(u)=1$; \ \  $k_{\emme}(v)=0$; \ \ $k_{\emme}(x_V^r)=0$ if $V \cap J= \emptyset$ and $k_{\emme}(x_V^r)=1$ otherwise;
 
$u \in P_h^I$, for all $h=1,\ldots,n$;  $u \in (H \sqcap D_1 \sqcap \ldots D_m \sqcap D_S)^I$;

$v \in P_h^I$ iff $p_h \not \in J$; \   $v \in (\neg H)^I$;

$x_V^r \in P_h^I$ iff $p_h \in V$;

if $k_{\emme}(x_V^r)=0$ then  $x_V^r \in  (\neg H)^I$;

if $k_{\emme}(x_V^1)=1$ then $x_V^1 \in  H^I$;

if $k_{\emme}(x_V^2)=1$ then $x_V^2 \in (\neg H)^I$;

\noindent
Finally, for all $y \in \Delta$ such that $y \neq u$,
for each $j=1,\ldots, m$, we let $y \in D_j^I$ iff some literal $l_i^j$ of $\gamma_j$ is true in $J$;
and, also, $y \in D_S^I$ iff $y$ is an instance of all $D_1, \ldots, D_m$.

It turns out that, for all the variables $p_h \in J$, $u$ is an instance of $\tip(P_h)$ while, 
for all the variables $p_h \not \in J$, $v$ is an instance of $\tip(P_h)$
and all domain elements are instances of $\exists U.(\tip(\top) \sqcap P_h)$.

It is easy to see that $\emme$ is a model of $K$, as it satisfies all the assertions and inclusions in $K$.
Also, $\emme$ is a canonical model of $K$, as all the non-extended concepts occurring in $K$
have an instance in the model. 
To see that $\emme$ falsifies $C_L(a)$, we proceed by cases. Consider the case where $L=p_h$ and $C_L=\tip(p_h)$.
In this case, $p_h \not \in J$ and, by construction, $a^I \not \in \tip(P_h)^I$, so that $\emme$ falsifies $C_L(a)$.
Consider the case when $L=\neg p_h$ and  $C_L=\exists U.(\tip(\top) \sqcap P_h)$. It must be that  $p_h  \in J$ and, by construction,
$a^I  \in \tip(P_h)^I$. 
As there is no domain element with rank $0$ which is an instance of $P_h$, 
$a \not \in (\exists U.(\tip(\top) \sqcap P_h))^I$. Hence, $\emme$ falsifies $C_L(a)$. 

The minimality of $\emme$ among the canonical models of $K$ can be shown by contradiction. Suppose that $\emme$ is not minimal, then there is a canonical model $\emme'$ $=\sx\Delta', <, \cdot^{I'}\dx$ of $K$
(with $\Delta'=\Delta$ and $I'=I$)  such that $\emme' \prec \emme$.
Hence there must be at least one concept $C$ which has rank 1 in $\emme$ and rank $0$ in $\emme'$.
In particular, some concept $P_h$, with $h=1,\ldots, n$, must have  rank 1 in $\emme$ and rank $0$ in $\emme'$.
In fact, the concept $H$ cannot have rank $0$ in any model of $K$ (by inclusion (1)), 
and the interpretation of all other concepts is determined by the interpretation of the $P_h$'s (given inclusions (2)-(5)).
Let $\{P_{j_1}, \ldots,P_{j_r} \}$ be the set of concepts having rank $1$ in $\emme$ and rank $0$ in $\emme'$, while all other concepts $P_h$  keep the same rank in $\emme'$ as in $\emme$. 
It is easy to see that, as $\emme'$ is a model of $K$, it is possible to construct from $\emme'$ a model $J'$ of $S$ such that $J' \subseteq J$, 
thus contradicting the hypothesis that $J$ is a minimal model of $S$.
\normalcolor
\end{proof}
\noindent
From the reduction above and the fact that minimal entailment for PDLP is
 \textsc{$\Pi^P_2$}-hard \cite{Eiter95}, it follows:
\begin{theorem}
Minimal entailment under minimal canonical model semantics is $\Pi^P_2$-hard.
\end{theorem}

It is an open issue whether a similar proof can be given also for simple knowledge bases
(as defined in section \ref{sec-rat}).
For simple KBs in $\alctr$, it was proved in \cite{AIJ15} that all minimal canonical models of the KB
assign the same ranks to concepts,
namely, the ranks determined by the rational closure construction. 
This result also holds for the fragment of $\sroelrt$ included in the language of $\alc$ plus typicality,
which however, does not contain nominals, role inclusions, and other constructs of $\sroel$.
The presence of the new constructs and of role inclusions makes alternative minimal canonical model arise 
with incomparable rankings for concepts. 
Even worst, a KB may have no canonical model.
This makes the adequacy of the minimal canonical model semantics disputable, and in \cite{TPLP2016}
an alternative $\tip$-minimal model semantics has been introduced, which coincides with the minimal canonical model semantics
under some conditions. 
\normalcolor
The existence of alternative minimal models for a KB with free occurrences of typicality in the propositional case was observed
in \cite{Booth15} for Propositional Typicality Logic (PTL).
%
The following is an example KB in $\sroelrt$ with alternative minimal canonical models having incomparable rank assignments.
\begin{example} \label{exa: estensioni multiple}
Let $\mathit{K=(TBox, RBox,ABox)}$, where $\mathit{RBox=ABox= \emptyset}$ and $TBox$ contains the inclusion  $\exists U.(A \sqcap \tip(\top)) \sqcap\exists U.(B \sqcap \tip(\top)) \sqsubseteq \bot$
meaning that it is not the case that an $A$ element and a $B$ element may have both rank 0.
Indeed, in the minimal canonical models of this knowledge base either $A$ has rank $0$ and $B$ rank $1$ or vice-versa.
\end{example}

\section{Deciding rational entailment in polynomial time}\label{Sec:encoding}  \label{sec:polyn}

 While  instance checking in $\sroelrt$ under minimal entailment is \textsc{$\Pi^P_2$}-hard, in this section
we prove that instance checking under rational entailment can be decided in polynomial time
for normalized KBs, by defining a translation of a normalized KB into a set of Datalog rules,
whose grounding is polynomial in the size of the KB.
In particular, we extend the Datalog materialization calculus for $\sroel$, 
proposed by Kr\"{o}tzsch \cite{KrotzschJelia2010},
to deal with typicality concepts and with instance checking under rational entailment  in $\sroelrt$.

The calculus in \cite{KrotzschJelia2010} uses predicates 
$\mathit{inst(a,C)}$ (whose meaning includes: the individual $a$ is an instance of concept name $C$,
see \cite{jeliaReport} for details),
$\mathit{triple(a,R,b)}$ ($a$ is in relation $R$ with $b$), $\mathit{self(a,R)}$
($a$ is in relation $R$ with itself).
To map a ${\cal SROEL}(\sqcap,$ $\times)^{\Ra}\tip$ KB to a Datalog program,
we add predicates to represent that: an individual $a$ is a typical instance of a concept name ($\mathit{typ(a,C)}$);
the ranks of two individuals $a$ and $b$ are the same ($\mathit{sameRank(a,b)}$);
the rank of $a$ is less or equal than the one of $b$ ($\mathit{leqRank(a,b)}$).

Besides the constants for individuals in $N_I$ (which are assumed to be finitely many), the calculus in \cite{KrotzschJelia2010}
exploits auxiliary constants
 $aux^{A\sqsubseteq \exists R.C}$ (one for each inclusion of the form $\mathit{A\sqsubseteq \exists R.C}$) to deal with existential restriction.
We also need to introduce an auxiliary constant $aux_C$
for any concept $\tip(C)$ occurring in the KB or in the query, used as a representative typical $C$, in case
$C$ is non-empty.

Given a normalized KB $\mathit{K = (TBox,RBox,ABox)}$
and query $Q$ of the form $C(a)$ or $\tip(C)(a)$,
where $C$ is a concept name in the normalized KB,
the Datalog program for instance checking 
 in  $\sroelrt$,
i.e.\ for querying whether$\mathit{K \models_{sroelrt} Q}$,
is a program $\Pi(K)$, the union of:
\begin{quote}
1. \ $\Pi_{K}$, the representation of $K$ as a set of Datalog facts, based on the input translation in \cite{KrotzschJelia2010};\\
2. \ $\Pi_{IR}$, the inference rules of the basic calculus in \cite{KrotzschJelia2010};\\
3. \ $\Pi_{RT}$, containing the additional rules for reasoning with typicality in \linebreak $\sroelrt$.
\end{quote}
A query $Q$ of the form $\tip(C)(a)$, or $C(a)$, is mapped to a goal $G_Q$ of the form
$\mathit{typ(a,C)}$, or $\mathit{inst(a,C)}$.
Observe that restricting queries to concept names is not a severe restriction as an arbitrary query $C(b)$ can be replaced by a query $A(b)$ with $A$ a new concept name,
by adding 
$C \sqsubseteq A$ to the TBox \cite{rifel} and, of course,
normalizing this inclusion when normalizing TBox.

We define $\Pi(K)$ in such a way that $G_Q$ is derivable in Datalog from $\Pi(K)$ (written $\Pi(K) \vdash G_Q$)  if and only if $\mathit{K \models_{sroelrt} Q}$.

$\Pi_K$ includes the result of the input translation in section 3 in \cite{KrotzschJelia2010}
where $\mathit{nom(a)}$, $\mathit{cls(A)}$, $\mathit{rol(R)}$ are used for
$\mathit{a \in N_I}$ , $\mathit{A \in N_C}$, $\mathit{R \in N_R}$, and, for example:
\begin{itemize}
\item
 $\mathit{subClass(a,C)}$, $\mathit{subClass(A,c)}$, $\mathit{subClass(A,C)}$ are used for $C(a)$, $A \sqsubseteq \{c\}$, $A \sqsubseteq C$;\\
\item
 $\mathit{subEx(R,A,C)}$ is used for $\exists R . A \sqsubseteq C$ ;
\end{itemize}
and similar statements represent other axioms in the normalized KB.

The following is the additional mapping
for the extended syntax of the \linebreak $\sroelrt$ normal form
(note that no mapping is needed for assertions $\tip(C)(a)$, as they do not occur in a normalized KB):
\vspace{-0.3cm}
\begin{tabbing}
$ \mathit{ \exists R . Self \sqsubseteq C }$ \= \kill  \\
\> $ \mathit{A\sqsubseteq \tip(B)} $  \' $\mapsto  \mathit{supTyp(A,B)} $ \\
\> $ \mathit{\tip(B) \sqsubseteq C} $ \' $\mapsto  \mathit{subTyp(B,C)} $
\end{tabbing}
Also, we need to add $\mathit{top}(\top)$ and $\mathit{cls}(\top)$ to the input translation,
as well as facts $\mathit{auxtc(aux_C,C)}$ to relate $aux_C$ constants to the corresponding concept
for all $C$s such that $\tip(C)$ occurs in the normalized KB, and, additionally, 
$\mathit{auxtc(aux_\top,\top)}$.

$\Pi_{IR}$ contains all the
inference rules from \cite{KrotzschJelia2010}\footnote{Here, $u, v, x, y, z, w$, possibly with suffixes, are variables.}:
\begin{tabbing}
$(10)$ \= \kill \\
$(1) ~ \mathit{inst(x, x) \leftarrow nom(x) }  $\\
$(2) ~ \mathit{self(x, v) \leftarrow  nom(x), triple(x, v, x) } $\\
$(3) ~ \mathit{inst(x, z) \leftarrow top(z), inst(x, z') } $\\
$(4) ~ \mathit{inst(x, y) \leftarrow bot(z), inst(u, z), inst(x, z'), cls(y) } $\\
$(5) ~ \mathit{inst(x, z) \leftarrow subClass(y, z), inst(x, y) }  $\\
$(6) ~ \mathit{inst(x, z) \leftarrow subConj(y1, y2, z), inst(x, y1), }$ $ \mathit{inst(x, y2) } $\\
$(7) ~ \mathit{inst(x, z) \leftarrow subEx(v, y, z), triple(x, v, x'), }$ $\mathit{inst(x', y) } $\\
$(8) ~ \mathit{inst(x, z) \leftarrow subEx(v, y, z), self(x, v), inst(x, y) } $\\
$(9) ~ \mathit{triple(x, v, x') \leftarrow supEx(y, v, z, x'), inst(x, y) } $\\
$(10) ~ \mathit{inst(x', z) \leftarrow supEx(y, v, z, x'), inst(x, y) } $\\
$(11) ~ \mathit{inst(x, z) \leftarrow subSelf(v, z), self(x, v) } $\\
$(12) ~ \mathit{self(x, v) \leftarrow supSelf(y, v), inst(x, y) } $ \\
$(13) ~ \mathit{triple(x, w, x') \leftarrow subRole(v, w), triple(x, v, x') } $ \\
$(14) ~ \mathit{self(x,w) \leftarrow subRole(v, w), self(x, v) }  $ \\
$(15) ~ \mathit{triple(x, w, x'') \leftarrow subRChain(u, v,w), } $ $ \mathit{triple(x, u, x'), triple(x', v, x'') } $ \\
$(16) ~ \mathit{triple(x, w, x') \leftarrow subRChain(u, v,w), self(x, u), } $ $ \mathit{triple(x, v, x') }  $ \\
$(17) ~ \mathit{triple(x, w, x') \leftarrow subRChain(u, v, w), } $ $ \mathit{triple(x, u, x'), self(x', v) }  $ \\
$(18) ~ \mathit{triple(x, w, x) \leftarrow subRChain(u, v,w), self(x, u),  } $ $ \mathit{self(x, v) }  $ \\
$(19) ~ \mathit{triple(x, w, x') \leftarrow subRConj(v1, v2,w),  }  $ $ \mathit{triple(x, v1, x'), triple(x, v2, x') }  $ \\
$(20) ~ \mathit{self(x,w) \leftarrow  subRConj(v1, v2,w), self(x, v1), }  $  $ \mathit{self(x, v2) } $ \\
$(21) ~ \mathit{triple(x, w, x') \leftarrow subProd(y1, y2,w), inst(x, y1), } $ $ \mathit{inst(x', y2) }  $ \\
$(22) ~ \mathit{self(x,w) \leftarrow  subProd(y1, y2,w), inst(x, y1), } $ $ \mathit{inst(x, y2) } $ \\
$(23) ~ \mathit{inst(x, z1) \leftarrow supProd(v, z1, z2), triple(x, v, x') } $ \\
$(24) ~ \mathit{inst(x, z1) \leftarrow supProd(v, z1, z2), self(x, v) } $ \\
$(25) ~ \mathit{inst(x', z2) \leftarrow  supProd(v, z1, z2), triple(x, v, x') } $ \\
$(26) ~ \mathit{inst(x, z2) \leftarrow  supProd(v, z1, z2), self(x, v) } $ \\
$(27) ~ \mathit{inst(y, z) \leftarrow  inst(x, y), nom(y), inst(x, z) } $ \\
$(28) ~ \mathit{inst(x, z) \leftarrow  inst(x, y), nom(y), inst(y, z) } $ \\
$(29) ~ \mathit{triple(z, u, y) \leftarrow inst(x, y), nom(y), triple(z, u, x) } $
\end{tabbing}
Note that ``statements $inst(a,b)$, with $a$ and $b$ individuals, encode equality of $a$ and $b$" \cite{jeliaReport}.

$\Pi_{RT}$, the set of rules to deal with typicality, 
contains rules for $\mathit{supTyp}$ and $\mathit{subTyp}$ axioms, and
rules that deal with the rank of domain elements.
In the rules, $x,y,z,Aux,$ $A,B,C$  are all Datalog variables.
\begin{tabbing}
$(CM)$ \= xxxxxxxxxxxxx \= \kill 
$(\mathit{SupTyp}) ~ \mathit{typ(x,z) \leftarrow supTyp(y,z),inst(x,y) } $ \\
$(\mathit{SubTyp}) ~ \mathit{inst(x,z) \leftarrow subTyp(y,z),typ(x,y)} $ \\
$(\mathit{Refl}) ~ \mathit{inst(x,y) \leftarrow typ(x,y) } $ \\
$(\mathit{A0})  ~ \mathit{typ(Aux,C)  \leftarrow  inst(x,C), auxtc(Aux,C) } $ \\
$(\mathit{A1})  ~ \mathit{leqRank(x,y)  \leftarrow  typ(x,B), inst(y,B) } $ \\  
$(\mathit{A2})  ~ \mathit{sameRank(x,y)  \leftarrow  typ(x,A), typ(y,A) } $ \\   
$(\mathit{A3})  ~ \mathit{ typ(x,B) \leftarrow  sameRank(x,y), inst(x,B), typ(y,B) } $\\  
$(\mathit{B1})  ~ \mathit{sameRank(x,z)  \leftarrow  sameRank(x,y), sameRank(y,z) } $ \\   
$(\mathit{B2})  ~ \mathit{sameRank(x,y)  \leftarrow  sameRank(y,x) } $ \\   
$(\mathit{B3})  ~ \mathit{leqRank(x,y)  \leftarrow  sameRank(y,x) } $ \\   
$(\mathit{B4})  ~ \mathit{leqRank(x,z)  \leftarrow  leqRank(x,y), leqRank(y,z) } $ \\   
$(\mathit{B5})  ~ \mathit{sameRank(x,y)  \leftarrow  leqRank(x,y), leqRank(y,x) } $ \\   $(\mathit{B6})  ~ \mathit{ sameRank(x,y) \leftarrow 
nom(y), inst(x,y) } $
\end{tabbing}

\noindent
Rules $\mathit{(SupTyp,SubTyp)}$ deal with the corresponding statements,
$\mathit{(Refl)}$ corresponds to the reflexivity property (see Section 2),
$\mathit{(A0-A3)}$ encode properties of ranked models:
if there is a $C$ element, there must be a typical $C$ element $\mathit{(A0)}$;
a typical $B$ element has a rank less or equal to the rank of any $B$ element $\mathit{(A1)}$;
two elements which are both typical $A$ elements have the same rank $\mathit{(A2)}$;
if $x$ is a B element and has the same rank as a typical $B$ element, $x$ is also a typical $B$ element $\mathit{(A3)}$.
Rules $\mathit{(B1-B6)}$ define properties of rank order.
In particular, by $\mathit{(B6)}$, two constants that correspond to the same domain element have the same rank.

Observe that the semantic properties of rational consequence relation introduced in Section 2 are enforced by the specification above.
Consider, for instance, $\mathit{(S\mbox{-}CM)}$.
Suppose that 
$\mathit{subTyp(A,B)}$ and $\mathit{subTyp(A,C)}$ are in $\Pi_K$
(as $\mathit{\tip(A) \sqsubseteq B}$, $\mathit{\tip(A) \sqsubseteq C}$ are in $K$)
and that $D$ is a concept name  defined to be equivalent to $A \sqcap B$ in $K$.
Suppose that $\mathit{typ(a,D)}$ holds. 
One can infer
$\mathit{typ(a,A)}$ 
and hence $\mathit{inst(a,C)}$, 
i.e., typical $A \sqcap B$'s inherit from typical $A$'s the property of being $C$'s
(the inference for $Paul$ in Example \ref{exa:1}).
In fact, $\mathit{typ(a,A)}$ is inferred showing that
$a$ (who is a typical $D$ and an $A$, as it is a $D$)
and $aux_A$ (who is a typical $A$, by $\mathit{(A0)}$, and a $D$, since all the typical $A$'s are also $B$'s and hence $A \sqcap B$'s)
have the same rank. 
In fact, using$\mathit{(A1)}$ twice, one can conclude both $\mathit{leqRank(a,aux_A)}$ and
$\mathit{leqRank(aux_A,a)}$ so that, by $\mathit{(B5)}$, $\mathit{sameRank(a,aux_A)}$.
Then, by $\mathit{(A3)}$, we infer $\mathit{typ(a,A)}$. With rule $\mathit{(subTyp)}$,
from $\mathit{typ(a,A)}$ and $\mathit{subTyp(A,C)}$, we conclude $\mathit{inst(a,C)}$.

Reasoning in a similar way, one can see that properties $\mathit{(S\mbox{-}RM)}$ and $\mathit{(S\mbox{-}LLE)}$ are also enforced by the rules above.
For $\mathit{(S\mbox{-}RM)}$:
from the fact that there is a domain element $a$ which is a
$\tip(A)$ and a $C$ element (i.e. $\mathit{typ(a,A)}$ and $\mathit{inst(a,C)}$ hold), and from the fact that there is a $b$
who is a typical $A \sqcap C$ element (i.e. that $\mathit{typ(b,D)}$ holds, for some concept $D$ equivalent to $A \sqcap C$),
we can conclude that $b$ is also a typical $A$ element (i.e. $\mathit{typ(b,A)}$ holds).
Inference in $\sroel$ already takes care of the semantic properties of conjunctive consequences $\mathit{(S\mbox{-}And)}$ and right weakening $\mathit{(S\mbox{-}RW)}$.

\begin{theorem}\label{ASP}
For a $\sroelrt$ KB in normal form $K$, and a query $Q$ of the form $\tip(C)(a)$ or $C(a)$,
$\mathit{K \models_{sroelrt} Q}$
if and only if  $\Pi(K) \vdash G_Q$.
\end{theorem}

\begin{proof} 
For completeness, we procede by contraposition, similarly to \cite{jeliaReport}, Lemma 3.
Assume that  $\mathit{inst(a,C)}$ (respectively, $\mathit{typ(a,C)}$) is not derivable from $\Pi(K)$.
Let $J$ be the minimal Herbrand model of the Datalog program $\Pi(K)$; then
$\mathit{inst(a,C)} \not \in J$ (resp. $ \mathit{typ(a,C)} \not \in J$).
From $J$ we  build a ranked model $\emme=(\Delta, <, \cdot^I)$ for $K$
such that  $C(a)$ (respectively, $\tip(C)(a)$) is not satisfied in $\emme$.
As in \cite{jeliaReport}, we can build the domain {\color{black} $\Delta$} of $\emme$ from the set $Const$ including all the name constants $c \in N_I$ occurring
in the ASP program $\Pi(K)$  
as well as all the auxiliary constants, then defining an equivalence relation $\approx$ over constants  as 
the reflexive, symmetric and transitive closure of the relation
$\{ (c,d) \mid inst(c,d) \in J$, for $c \in Const$ and  $d \in N_I\}$.
Let $[c]=\{ d \mid d \approx c \}$ denote the equivalence class of $c$;
we define the domain $\Delta$ of the interpretation $\emme$ as
the set including the equivalence classes $[c]$ for the $c \in N_I$ and, possibly, additional domain elements for auxiliary constants:
$\Delta=  \{[c] \mid c \in N_I\} \cup \{\wARC_1, \wARC_2 \mid$ $inst(\auxARC,e) \in J$ for some $e$ and there is no $d \in N_I$ such that $\auxARC \approx d\}$
 $ \cup \{z_C^1 , z_C^2 \mid$ $inst( aux_C,e) \in J$ for some $e$ and there is no $d \in N_I$ such that $aux_C \approx d\}$.
Two copies of auxiliary constants are introduced, as in \cite{jeliaReport}, to handle $\mathit{Self}$ statements.

For each element $e \in \Delta$, we define a projection $\prj(e)$ to $Const$ as follows:
(i) $\prj([c])=c$;
(ii) $\prj(\wARC_i)=\auxARC$, i=1,2;
(iii) $\prj(z_C^i)=aux_C, i=1,2$. 
$J$ contains all the details about the interpretation of concepts and roles, from which an interpretation $\emme$ can be defined
as follows:\\
- $\forall c \in N_I$, \ $c^I =[c]$;\\
- $\forall d \in \Delta$, \   $d \in A^{I}$ iff $\mathit {inst}(\prj(d), A) \in J$; \\
- $\mathit{\forall d, e \in \Delta}$, \ $\mathit{(d,e) \in R^{I}}$ iff ($\mathit{triple(\prj(d), R, \prj(e)}) \in J$ and  $\mathit{d \neq e}$) or ($\mathit{self(\prj(d), R) \in J}$ and  $\mathit{d = e}$).

We define a relation $<'$ over {\color{black} constants}, 
letting $x<'y$ iff there is a concept name $C$, s.t.
$ \mathit{typ(x,C)}$, $\mathit{inst(y,C)} \in J$ and  $ \mathit{typ(y,C)} \not \in J$;
 we can show that {\color{black} its transitive closure $<^+$} is a strict partial order.
To see that $<^+$ is irreflexive, we procede by contradiction.
If there were a chain $x<' y_1 < \ldots <' y_n = x $, so that  $x <^+ x$,
by definition of $<'$, the facts $\mathit{leqRank(x,y_1)}, \ldots$, $\mathit{leqRank(y_{n-1},x)}$
would be in $J$ and, by applications of (B4) and (B5), $\mathit{sameRank(x,y_1)}$ would also be in $J$.
As $x<' y_1$, there is some concept name $C$, s.t.
$ \mathit{typ(x,C)}, \mathit{inst(y_1,C)} \in J$ and  $ \mathit{typ(y_1,C)} \not \in J$,
but, from $ \mathit{typ(x,C)}, \mathit{inst(y_1,C)}$ and $\mathit{sameRank(x,y_1)}$, by (A3),
$ \mathit{typ(y_1,C)}$ must be in $J$ as well, a contradiction.
Similarly, one can see that $<^+$ is compatible with
the $\mathit{sameRank}$ 
predicate in $J$ (i.e., it is not the case that $a<^+b$ and $sameRank(a,b)$) 
and with the $\approx$ equivalence relation between constants
(i.e., it is not the case that $a<^+b$ and $a \approx b$)
so that 
$<^+$ can be extended to a modular {\color{black} partial } order over the domain $\Delta$ as follows.
\color{black}
First, a partial ordering  over elements in $\Delta$ is defined, letting 
$e < d$ iff $\prj(e)<^+\prj(d)$, for all $e,d \in \Delta$
(where the definition does not depend on the choice of the representative element in a class).
%
Then the elements in $\Delta$ are partitioned into the sets $Rank_0, \ldots, Rank_n$,
where $Rank_i$ (the set of domain elements of rank $i$) is defined by induction on $i$,  as follows:
$Rank_0$ contains all the elements $x \in \Delta$ such that there is no $y \in \Delta$ with $y<x$;
$Rank_i$ contains all the elements $x \in \Delta - (Rank_0 \cup \ldots \cup Rank_{i-1})$
such that there is no $y \in \Delta - (Rank_0 \cup \ldots \cup Rank_{i-1})$ with $y<x$.
We let $n$ be the least integer such that  $\Delta - (Rank_0 \cup \ldots \cup Rank_n)= \emptyset$.
In $\emme$ we let $k_{\emme}(x)= i$ if $x \in Rank_i$, for all $x \in \Delta$ and $i=1,n$.
\normalcolor

It can be shown that $\emme$ is a model of $K$ and it does not satisfy $C(a)$ (respectively, $\tip(C)(a)$).
The proof that $\emme$ is a model of $K$, i.e. it satisfies all the
axioms in KB, is the same as in  \cite{jeliaReport} (see Lemma 2),
apart from the fact that we have to consider the additional axioms $A \sqsubseteq \tip(B)$ and $ \tip(B) \sqsubseteq C$
(observe that an assertion $T(C)(a)$ cannot occur in the ABox, as $K$ is in normal form). 

For $A \sqsubseteq \tip(B)$ in $K$, we have $\mathit{supTyp(A,B)} \in J$.
Let us assume that $d \in A^I$. We want to prove that $d \in (\tip(B))^I$.
By construction  $\mathit {inst}(\prj(d), A) \in J$.
By rule {\em (SupTyp)}, $\mathit {typ}(\prj(d), B) \in J$.
By rule {\em (Refl)}, $\mathit {inst}(\prj(d), B) \in J$, i.e., $d \in B^I$.
Let us assume $k_\emme(d)=h$.

To show that $d$ is a typical $B$,
we have to show that, for all the domain elements $e$ with rank $j<h$,
$e \not \in B^I$.
Suppose, by contradiction, that there were an $e$ such that  $e \in B^I$ and  $k_\emme(e) < k_\emme(d)$.
By construction, it would be that  $\mathit {inst}(\prj(e), B) \in J$ and that $\prj(e) <^+ \prj(d)$.
Hence, there would be a sequence  $\prj(e) <' y_1 < \ldots <' y_n = \prj(d) $ ($n\geq 1$)
and, by definition of $<'$, $\mathit{leqRank(\prj(e), y_1)}$, $\ldots$, $\mathit{leqRank(y_{n-1},\prj(d))}$ would be in $J$.
By (B4) $\mathit{leqRank(\prj(e), \prj(d))}$ would be in $J$.
Also, from $\mathit {inst}(\prj(e), B) \in J$ and $\mathit {typ}(\prj(d), B) \in J$, by (A1), $\mathit{leqRank(\prj(d), \prj(e))} \in J$.
Therefore, by (A2), $\mathit{sameRank(\prj(e),}$  $\mathit{\prj(d))}$ would be in $J$.
This contradicts the fact that $\prj(e) <^+ \prj(d)$ (as $<^+$ must be compatible with {\em sameRank}), and hence it contradicts the hypothesis that $k_\emme(e) < k_\emme(d)$.

For $\tip(B) \sqsubseteq C$ in $K$, we have $\mathit{subTyp(B,C)} \in J$.
Let $d \in (\tip(B))^I$. We have to prove that $d \in C^I$.
Let $k_\emme(d)=h$. 
As $d \in (\tip(B))^I$, $d \in B^I$ and, for all $e \in \Delta$ such that $k_\emme(e)=j<h$, $e \not \in B^I$.
From $d \in B^I$, by the definition of $\emme$, $\mathit {inst}(\prj(d), B) \in J$.

In the case $\mathit {typ}(\prj(d), B) \in J$, by {\em (SubTyp)} $\mathit {inst}(\prj(d), C) \in J$
and, by definition of $\emme$, $d \in C^I$.
We show that the case when $\mathit {typ}(\prj(d), B) \not \in J$ cannot occur.

Suppose by contradiction that $\mathit {typ}(\prj(d), B) \not \in J$.
Consider the auxiliary constant $aux_B$ and let $aux_B^I \in Rank_r$, i.e.  $k_\emme(aux_B^I)=r$.
By rule (A0), from $\mathit {inst}(\prj(d), B) \in J$, we have $\mathit {typ}(aux_B, B) \in J$
and, from $\mathit{ (Refl)}$, $\mathit {inst}(aux_B, B) \in J$.
From $\mathit {typ}(aux_B, B) \in J$ and $\mathit {inst}(\prj(d), B) \in J$, 
as we have assumed $\mathit {typ}(\prj(d), B) \not \in J$, by definition of $<'$, we have $aux_B <' \prj(d)$, and hence  $aux_B <^+ \prj(d)$.
Therefore, by construction of $\emme$, $aux_B^I < d$. 
Also from $\mathit {inst}(aux_B, B) \in J$, by construction, we have $aux^I \in B^I$, which contradicts the hypothesis that $d \in (\tip(B))^I$.

We have proved that $\emme$ is a model of KB.
To conclude the proof, let us consider the query $Q$.
If $Q=C(a)$, $\mathit {inst}(a, C) \not \in J$ then, by definition of $\emme$, $a^I \not \in C^I$, so that $C(a)$ is not satisfied in $\emme$.
If $Q=\tip(C)(a)$, $\mathit {typ}(a, C) \not \in J$. We prove that $a^I \not \in (\tip(C))^I$.
We consider 2 cases: $\mathit {inst}(a, C) \not \in J$ and $\mathit {inst}(a, C)  \in J$.
In the first case, by definition of $\emme$, $a^I \not \in C^I$ and hence $a^I \not \in (\tip(C))^I$.
In the second case, from $\mathit {inst}(a, C) \in J$, by (A0) $\mathit {typ}(aux_C, C)  \in J$.
As $\mathit {typ}(a, C) \not \in J$, by definition of $<'$ we have: $aux_C <' a$.
Hence, $aux_C <^+ a$ and, in $\emme$, $aux_C^I <a^I$. As by (Refl) $\mathit {inst}(aux_C, C)  \in J$,
in $\emme$ $aux_B^I \in C^I$. Therefore, $a^I \not \in (\tip(C))^I$.
This concludes the proof of Completeness.

Proving the soundness of the Datalog encoding requires
showing that, if  $\Pi(K) \vdash G_Q$, for
a query $Q$ of the form $\tip(C)(a)$ or $C(a)$,
then, $Q$ is a logical consequence of $K$.
The proof is similar to the proof of Lemma 1 in \cite{jeliaReport}.
First we associate to each constant $c$ of the Datalog program $\Pi(K)$ a concept expression $\kappa(c)$  as follows:

if $c \in N_I$ then $\kappa(c)= \{c\}$;

if $c = aux^\alpha$, for $\alpha=A \sqsubseteq \exists R. B$,  then $\kappa(c)= B \sqcap \exists R^-.A$;

if $c = aux_C$,  then $\kappa(c)= \tip(C)$.\\
The following statements can be proved by induction on the height of the derivation tree of each atom from the program $\Pi(K)$:

- if $\Pi(K) \vdash\mathit{inst(c,A)}$, for $A \in N_C$, then $K \models_{sroelrt} \kappa(c) \sqsubseteq A$;

- if $\Pi(K) \vdash\mathit{inst(c,d)}$, for $d \in N_I$, then $K \models_{sroelrt} \kappa(c) \sqsubseteq \{d\}$;

- if $\Pi(K) \vdash\mathit{triple(c,R,d)}$, then $K \models_{sroelrt} \kappa(c) \sqsubseteq \exists R. \kappa(d) $;

- if $\Pi(K) \vdash\mathit{self(c,R)}$, for $A \in N_C$, then $K \models_{sroelrt} \kappa(c) \sqsubseteq \mathit{\exists R. Self}$;

- if $\Pi(K) \vdash\mathit{typ(c,A)}$, then $K \models_{sroelrt} \kappa(c) \sqsubseteq \tip(A)$;

- if $\Pi(K) \vdash\mathit{sameRank(c,d)}$ then for all models $\emme$ of $K$, $k_{\emme}(x)=k_{\emme}(y)$, $\forall x \in \kappa(c)$,  $\forall y \in \kappa(d)$;

- if $\Pi(K) \vdash\mathit{leqRank(c,d)}$ then, for all models $\emme$ of $K$, $k_{\emme}(x) \leq k_{\emme}(y)$, $\forall x \in \kappa(c)$,  $\forall y \in \kappa(d)$;

\noindent
where, in all the items above, $\kappa(c)$ and $\kappa(d)$ are nonempty in all the models of $K$.

The proof of the first four items goes through as in the proof of Lemma 1 in \cite{jeliaReport},
with few additions for the first case, as there are additional inference rules (namely {\em (SubTyp)} and {\em (Refl)}) to derive $\mathit{inst(c,A)}$.
As an example, let us consider the case when  $\mathit{inst(c,A)}$ is derived from $\Pi(K)$ by applying rule {\em (SubTyp)}.
Then $\mathit{SubTyp(B,A)}$ and $\mathit{typ(c,B)}$ must be derivable and the height of their derivation tree is lower.
By inductive hypothesis, $K \models_{sroelrt} \kappa(c) \sqsubseteq \tip(B)$. 
Also, as $\mathit{SubTyp(B,A)}$ is in the input translation, $\tip(B) \sqsubseteq A$ must be in $K$.
Therefore, it follows that $K \models_{sroelrt} \kappa(c) \sqsubseteq A$.

If $\mathit{inst(c,A)}$ is derived from $\Pi(K)$ by applying the rule {\em (Refl)},
then $\mathit{typ(c,A)}$ is derived with a lower derivation tree. By inductive hypothesis, $K \models_{sroelrt} \kappa(c) \sqsubseteq \tip(A)$
and, hence, $K \models_{sroelrt} \kappa(c) \sqsubseteq A$.

Let us now consider the fifth item. Assume that $\Pi(K) \vdash\mathit{typ(c,A)}$;
$\mathit{typ(c,A)}$ can be derived using one of the rules {\em (SupTyp), (A0) or (A3)}. 
The first two cases are trivial.
If rule {\em(A3)} is used to derive $\mathit{typ(c,A)}$, then $\mathit{sameRank(c,d), inst(c,A), typ(d,A)}$
are derived  with a lower derivation tree. By inductive hypothesis, $K \models_{sroelrt} \kappa(c) \sqsubseteq A$, \ \ 
$K \models_{sroelrt} \kappa(d) \sqsubseteq \tip(A)$ and, for all the models $\emme$ of $K$,
the rank of all instances of $\kappa(c)$ is the same as the rank of all instances of $\kappa(d)$.
As all the $\kappa(c)$ are $A$ elements and have the same rank as the 
typical $A$ elements, then 
$K \models_{sroelrt} \kappa(c) \sqsubseteq \tip(A)$.
\end{proof}

$\Pi(K)$ contains a polynomial number of rules and exploits a polynomial number of concepts in the size of $K$,
hence instance checking in $\sroelrt$ can be decided in polynomial time using the calculus in Datalog.
The encoding can be processed, e.g., in an ASP solver such as Clingo or DLV
(with the proper capitalization of variables);
computation of the (unique, in this case)
answer set takes a negligible time for KBs with a hundred assertions (half of them with $\tip$).

Exploiting the approach presented in \cite{KrotzschJelia2010},  a version of the Datalog specification
where predicates $\mathit{inst}$, $\mathit{typ}$, $\mathit{triple}$, $\mathit{self}$, 
$\mathit{sameRank}$, $\mathit{leqRank}$
have an additional parameter (and is therefore less efficient than the previous one, although polynomial)
can be used to check subsumption for $\sroelrt$.
Indeed, to check a subsumption $A \sqsubseteq B$ from a KB $K$,
one can check that, for some new individual name  $c$,  $B(c)$ is entailed from the KB $K \cup \{A(c)\}$.
Following \cite{KrotzschJelia2010} the class $B$ can be used 
as the new individual name $c$.
In the ternary predicate $\mathit{inst(a,C,B)}$, the additional parameter $B$ stands for the assumption that the new individual $B$ is an instance of the concept $B$. This is encoded by adding the rule $\mathit{inst(B,B,B)} \leftarrow \mathit{cls(B)}$ to the inference rules (and similarly for the other predicates).

\section{Rational Closure in Datalog plus Stratified Negation}

Given a KB $K$, the model constructed in the completeness proof from the minimal Herbrand Model $J$ of the encoding $\Pi(K)$
is not in general a minimal model of $K$.
Consider the following example. 

\begin{example} \label{exa:sec5}
Let $K=(TBox, RBox, ABox)$, with $\mathit{RBox}=\emptyset$,
$\mathit{ABox}=\{ \tip(D) (a) \}$,
and $\mathit{TBox}$ containing the following inclusions:
(1) $\tip(A) \sqcap \tip(B) \sqsubseteq \bot$, \ 
%
(2) $ A \sqcap B  \sqsubseteq  D$, \ 
(3) $ D \sqsubseteq A $, \ 
(4) $ D \sqsubseteq  B$.
From $\Pi(K)$ we can derive $samerank(aux_D,a)$ and,
from the definition of $<$ we get: $aux_{\top} < aux_A$
(in fact, $\mathit{typ(aux_{\top},\top)}$ and $\mathit{inst(aux_A,\top)}$ are in $J$ while $\mathit{typ(aux_A,\top)\not \in J}$),
$aux_{\top} < aux_B$, $aux_A < aux_D$, $aux_B < aux_D$, $aux_A < a$, $aux_B < a$.
Using the construction in the completeness proof, we would let $\Delta= \{ [aux_{\top}]$, $[aux_A]$, $[aux_B]$, $[aux_D]$, $[a] \}$,
$Rank_0= \{[aux_{\top}] \}$;   $Rank_1= \{[aux_A], [aux_B] \}$ and $Rank_2= \{[aux_D], [a] \}$.
This is not a minimal model of $K$. In fact, we can lower the rank of $[aux_A]$ and $[aux_B]$ to $0$, without changing the evaluation of 
atomic concepts, roles and individual constants (as defined in the model, and in $J$) and we still obtain a model of $K$.
Similarly, the rank of $[aux_D]$ and $[a]$ can be lowered to $1$. The resulting model is a minimal model of $K$.
\end{example}

For simple $\sroelrt$ knowledge bases, i.e., for KBs where the typicality operator only occurs on the left hand side of inclusions in the Tbox, and we assume moreover it does not occur in the Rbox nor in the Abox,
the rational closure of TBox can be obtained,
adapting the construction in \cite{AIJ15} (Definitions 21 and 23).
Such a construction can be reformulated replacing the exceptionality check and the entailment in $\alctr$ with
the ones in $\sroelrt$.
The idea is that of assigning the lowest possible rank to each concept.
A concept with rank $0$ must have instances which do not violate any defeasible inclusions.
If a concept $C$ has a rank higher than $0$, then all its instances necessarily violate some defeasible inclusions.
This concept is {\em exceptional} with respect to $K$, as it cannot be satisfied by the most normal domain elements
in some (minimal and canonical) model of $K$.
The iterative construction of the rational closure builds on this notion of exceptionality, by repeatedly selecting exceptional concepts and their typicality inclusions.

Formally, in \cite{AIJ15}, a concept  $C$ is defined as exceptional for a knowledge base $K$ iff $\tip(\top) \sqsubseteq \neg C$ is rationally entailed by $K$.
As negation is not available in $\sroelrt$,  we reformulate exceptionality as follows:
 {\em a concept $C$ is exceptional for a knowledge base $K$} iff $K \models_{sroelrt} \tip(\top) \sqcap C \sqsubseteq \bot$.

An inclusion $\tip(C) \sqsubseteq D$ is {\em exceptional} for $K$ if $C$ is exceptional for $K$.
The set of inclusions that are exceptional for $K$ is denoted as ${\cal E}(K)$.

Given $K = (Tbox,Rbox,Abox)$, a sequence $E_0=K, E_1, \ldots , E_n$ of knowledge bases can be defined letting, for $i>0$:
$E_i = (Tbox_i,$ $Rbox, Abox)$, where $Tbox_i = {\cal E}(E_{i-1}) \cup \{ C \sqsubseteq D \in Tbox, \tip$ does not occur in $C \} $.
For all $i$, $Tbox_{i-1} \supseteq Tbox_i$. Being $K$ finite, there is a least finite $n \geq 0$ such that, for all $m>n$, $Tbox_m = Tbox_n$ (as special case, they are all $\emptyset$). The construction can then be limited to such $n$, which is reached if $Tbox_{n+1} = Tbox_n$.

A concept $C$ has {\em rank} $i$, i.e., $rank(C)=i$, if $i$ the least number such that $C$ is not exceptional for $E_i$. If $C$ is exceptional for all $E_i$, we set $rank(C) = \infty$.

The {\em rational closure of Tbox} for a knowledge base $K=(Tbox,Rbox,Abox)$ is defined as:
\ 
$\overline{Tbox} = \{ \tip(C) \sqsubseteq D | ~ \mathrm{either} ~ rank(C) < rank (C \sqcap \neg D) ~ 
\mathrm{or} ~ rank(C) = \infty \}
\cup \{ C \sqsubseteq D | K \models_{sroelrt}  C \sqsubseteq D \}$.

Actually, $C \sqcap \neg D$ is not in the language of $\sroelrt$, but as we shall see,
the condition involving it will be reformulated accordingly.  
For instance, in Example \ref{exa:1}, the rational closure construction would assign rank $0$ to concepts $\mathit{Student}$,  $\mathit{Student \sqcap Italian}$, $\mathit{Student \sqcap Young}$, $\mathit{MathHater}$, 
 rank $1$ to $\mathit{Student \sqcap \neg MathHater}$, $\mathit{Student \sqcap Nerd}$, and rank 2 to $\mathit{Student \sqcap Nerd \sqcap \neg MathLover}$. 

In the following we define rules for computing the rational closure construction, using a variant of the calculus in section \ref{sec:polyn} (similar to its variant for subsumption) and stratified negation.
The set of rules to infer whether inclusions of the form
$\tip(C) \sqsubseteq D$, with $C, D \in {N_C}$, 
are in the rational closure of the TBox for a simple $\sroelrt$ knowledge base $K,$ is a set $\Pi_{RT}(K)$ which is the union of:
\begin{itemize}
\item[1] \ $\Pi_K$ and $\Pi_{IR}$ as in section \ref{sec:polyn};
	\item[2]
	\ $\Pi'_{RT}$ which is  $\Pi_{RT}$ in section  \ref{sec:polyn} omitting rule $(\mathit{SupTyp})$, given that $K$ is a simple KB;
	\item[3]
	 \ $\Pi_{RC}$ which contains the rules for rational closure as described below.
\end{itemize} 

$\Pi_{RC}$ contains rules to define exceptionality of a concept $C$ wrt a KB $E_i$, which, 
as said before, means $E_i \models_{sroelrt} \tip(\top) \sqcap C \sqsubseteq \bot$,
for all the concepts $C$ such that $\tip(C)$ occurs in the program or in a typicality subsumption $ \tip(C) \sqsubseteq D$ to be checked.
For all such $Cs$ we assume to have a statement $\mathit{auxtc(aux_C,C)}$, and we
define the set of such $Cs$ with a rule 
$\mathit{t\_cls(C) \leftarrow auxtc(Aux,C)	} $.
\normalcolor
The entailment $E_i \models_{sroelrt} \tip(\top) \sqcap C \sqsubseteq \bot$
is verified checking whether, for a new constant $C$ (the same name of the class is reused 
as constant name, as in the calculus for subsumption in \cite{KrotzschJelia2010} mentioned at the end of section
\ref{sec:polyn}), it holds that 
$E_i \cup \{\tip(\top)(C), C(C) \} \models_{sroelrt} \bot (C)$.
Similarly to the calculus for subsumption, we use a version of the basic calculus defining predicates
$\mathit{inst\_h, self\_h, triple\_h, typ\_h, leqRank\_h, sameRank\_h }$, 
which correspond to the ones in 
the basic calculus in section \ref{sec:polyn}, with two additional parameters, one for the concept $C$ in the hypotheses $\tip(\top)(C), C(C)$ and one for an integer $i$ identifying $E_i$ in the sequence of KBs. Then, the idea is that $\mathit{inst\_h(y,z,C,i)}$ holds when
 $E_i \cup \{\tip(\top)(C), C(C) \} \models_{sroelrt} z(y)$.

$\Pi_{RC}$ includes the following rules for determining the exceptionality of concepts for each $E_i$:
\begin{tabbing}
	$(\mathit{C0})  ~ \mathit{t\_cls(C) \leftarrow auxtc(Aux,C)	} $\\
	$(\mathit{C1})  ~ \mathit{possrank(0..N) \leftarrow upperbound(N)	} $ \\
	$(\mathit{C2})  ~ \mathit{exceptional(C,I) \leftarrow t\_cls(C),possrank(I),cls(Z),inst\_h(C,Z,C,I),bot(Z)} $ \\
	$(\mathit{C3})  ~ \mathit{subTyp(C,D,0) \leftarrow subTyp(C,D)} $ \\
	$(\mathit{C4})  ~ \mathit{subTyp(C,D,I) \leftarrow possrank(I),subTyp(C,D,I-1),exceptional(C,I-1)} $ \\
	$(\mathit{C5})  ~ \mathit{typ\_h(C,\top,C,I) \leftarrow t\_cls(C),possrank(I)} $ \\
	$(\mathit{C6})  ~ \mathit{inst\_h(C,C,C,I) \leftarrow t\_cls(C),possrank(I)} $ 
\end{tabbing}
where: $(\mathit{C1})$ defines the possible rank values for concepts, given that an assertion 
$\mathit{upperbound(n)}$ is added to the input translation $\Pi(K)$, where $\mathit{n}$ is
one more the number of $\tip(C)$ occurring in $\mathit{Tbox}$;
$(\mathit{C2})$ defines exceptionality of $C$ for $E_i$;
$(\mathit{C3,C4})$ define which inclusions $\tip(C) \sqsubseteq D$ belong to ${\cal E}(E_{i-i})$
and then to $E_i$;
$(\mathit{C5,C6})$ provide the assumptions  $\{\tip(\top)(C), C(C) \}$
for reasoning in  $E_i \cup \{\tip(\top)(C), C(C) \}$.

Additionally, $\Pi_{RC}$ contains a version of the rules in $\Pi'_{RT}$ where predicates
$\mathit{inst, }$ $\mathit{self, triple, typ,}$ $\mathit{leqRank, sameRank}$
are replaced by
$\mathit{inst\_h, self\_h, triple\_h, typ\_h,}$ 
$\mathit{ leqRank\_h, sameRank\_h }$ respectively, with two additional
parameters $\mathit{D}$ and $\mathit{I}$; in all rules 
$\mathit{t\_cls(D),possrank(I)}$ is added to the antecedent,
and the rule derived from $(\mathit{subTyp})$ is:
\begin{tabbing}
	$(\mathit{subTypRC})  ~ \mathit{inst\_h(X,C,D,I) \leftarrow t\_cls(D),possrank(I),subTyp(A,C,I),}$\\
	\hspace{4cm}\ \ \ \ \ \ \ \ \ \ \ \ \ \ \ $\mathit{typ\_h(X,A,D,I)} $
\end{tabbing}
i.e., it uses $\mathit{subTyp(A,C,I)}$ rather than $\mathit{subTyp(A,C)}$, given that reasoning is
in $E_i$, not in $K$.

Let  $\Pi^{Pos}_{RT}(K)$ be the set of all the rules in $\Pi_{RT}(K)$ introduced so far, which are all positive Datalog rules.
It is easy to see that the following proposition holds:

\begin{proposition} \label{correttezza_RC}
For a $\sroelrt$ KB in normal form $K$: 
\begin{itemize}
\item[(1)]
$\Pi^{Pos}_{RT}(K) \vdash \mathit{subTyp(C,D,i)}$ iff  $\mathit{\tip(C) \sqsubseteq D \in E_i}$.
\item[(2)]
$\Pi^{Pos}_{RT}(K) \vdash \mathit{inst\_h(C,D,C,i)}$ iff  $\mathit{E_i \models_{sroelrt} \tip(\top) \sqcap C \sqsubseteq D}$ ; 
\item[(3)]
$\Pi^{Pos}_{RT}(K) \vdash \mathit{exceptional(C,i)}$ iff  $\mathit{E_i \models_{sroelrt} \tip(\top) \sqcap C \sqsubseteq \bot}$ (i.e., $C$ is exceptional for $E_i$)
\end{itemize}
\end{proposition}
\begin{proof} (Sketch) The proof can be done by induction on $i$. For $i=0$, (1) holds, as by rule (C3) $\mathit{subTyp(C,D,0)}$ is derivable when 
$\mathit{subTyp(C,D)}$ is derivable, i.e, by the input translation, when 
$\mathit{\tip(C) \sqsubseteq D \in K}$. Also, (C3) is the only rule to define $\mathit{subTyp(C,D,0)}$, which is not derivable if 
$\mathit{subTyp(C,D)}$ is not (that is when $\mathit{\tip(C) \sqsubseteq D \not \in K}$). Also, by construction, $E_0=K$. 
Item (2) follows according to the same considerations
given for the calculus for subsumption in \cite{KrotzschJelia2010}. As a difference, here the further hypothesis $\tip(\top)(C)$ 
is added besides the hypothesis $C(C)$ (see rules (C5) and (C6)). 
For item (3), $\mathit{exceptional(C,i)}$ is derivable by rule (C2) if and only if for some $Z$, $ \mathit{inst\_h(C,Z,C,i)}$ and $bot(Z)$ are derivable.
In this case, by (2), $\mathit{E_i \models_{sroelrt} \tip(\top) \sqcap C}$ $\mathit{ \sqsubseteq Z}$. Furthermore, as $bot(Z)$ is derivable,  $\mathit{ Z \sqsubseteq \bot}$
must be in $K$, and hence
$\mathit{E_i \models_{sroelrt}}$ $\mathit{ \tip(\top) \sqcap C \sqsubseteq \bot}$. 
The inductive case can be easily proved.
\end{proof}

$\Pi_{RC}$ also contains the following rules for computing the rank of concepts:
\begin{tabbing}
	$(\mathit{C7})  ~ \mathit{rank(C,0) \leftarrow t\_cls(C),not ~ exceptional(C,0)} $ \\
	$(\mathit{C8})  ~ \mathit{rank(C,I) \leftarrow t\_cls(C),possrank(I),exceptional(C,I-1),}$\\
	\hspace{2.9cm} $\mathit{ not ~ exceptional(C,I)} $ \\
	$(\mathit{C9})  ~ \mathit{newNonEx(I) \leftarrow t\_cls(C), rank(C,I)} $ \\
	$(\mathit{C10})  ~ \mathit{fixp(I) \leftarrow possrank(I),I>0,not ~ newNonEx(I)} $ \\
	$(\mathit{C11})  ~ \mathit{fixp(I) \leftarrow possrank(I), fixp(I-1)} $ \\
	$(\mathit{C12})  ~ \mathit{inf\_rank(C) \leftarrow fixp(I), exceptional(C,I)} $
\end{tabbing}
where, according to the definition of rank, a concept $C$ has rank $0$ if $C$ is not exceptional for $E_0$, i.e., by Proposition  \ref{correttezza_RC}, if
$\mathit{exceptional(C,0)}$ is not derivable (rule (C7)), and similarly for a rank $i>0$ (rule (C8)).
Predicate $newNonEx(i)$ is true if there is some new non exceptional concept at $i$, which was exceptional at $i-1$ 
(i.e. there is some concept with rank $i$). If not, the iteration has reached a fixpoint at $i$ (and $\mathit{fixp(i)}$ holds by (C10)), and the remaining  concepts $C$ which are exceptional for $E_i$, for some $i$ such that $\mathit{fixp(i)}$ holds, have an infinite rank and, by (C12), $\mathit{inf\_rank(C)}$ is derivable.
Clearly it cannot be the case that both $\mathit{inf\_rank(C)}$ and $\mathit{rank(C,i)}$ hold for some $i$.

Finally, for defeasible subsumption queries, we introduce the new predicate $\mathit{inrc}$ $\mathit{(x,y)}$
and the rules:
\begin{tabbing}
	$(\mathit{inrc1})  ~ \mathit{inrc(C,D) \leftarrow t\_cls(C),cls(D),rank(C,I), inst\_h(C,D,C,I)  } $ \\ 
	$(\mathit{inrc2})  ~ \mathit{inrc(C,D) \leftarrow t\_cls(C),cls(D),inf\_rank(C)} $
\end{tabbing}
\vspace{-0.1cm}
which determine whether $\tip(C) \sqsubseteq D$, with $C, D \in {N_C}$, 
is in $\mathit{\overline{Tbox}}$, the rational closure of Tbox for the knowledge base.
Rule $\mathit{(inrc1)}$ is for the case $rank(C) < rank (C \sqcap \neg D)$ and deserves some comment.
Given that $rank(C)=i$, it should be checked that $rank (C \sqcap \neg D) > i$.
$rank (C \sqcap \neg D)$ could either be equal to $i$ or larger. It is equal to $i$ iff $C \sqcap \neg D$ is not exceptional in $E_i$ and, then,
it is $> i$ iff $C \sqcap \neg D$ is exceptional in $E_i$, i.e.,
$E_i \models_{sroelrt} \tip(\top) \sqcap C \sqcap \neg D \sqsubseteq \bot$, i.e.,
$E_i \models_{sroelrt} \tip(\top) \sqcap C \sqsubseteq D$, which, by Proposition  \ref{correttezza_RC},  can indeed be verified by checking that 
$\mathit{inst\_h(C,D,C,I)}$ is derivable.

To avoid the derivation of an atom $\mathit{inrc(C,D)}$, for each defeasible inclusion $\tip(C) \sqsubseteq D$ belonging to the rational closure of $K$, for all $C, D \in {N_C}$, a condition  $\mathit{def\_subs(C,D)}$ can be introduced in the antecedent of rules $(\mathit{inrc1})$ and $(\mathit{inrc2})$, and
facts $\mathit{def\_subs(C,D)}$ added for the subumptions $\tip(C) \sqsubseteq D$ we are interested in checking,
as well as the fact $t\_cls(C)$, in case $\tip(C)$ does not occur in $K$.

Observe that the resulting program $\Pi_{RC}$ contains negated literals and is {\em stratified}  \cite{Przymusinski91},
i.e., the predicates in the program can be partitioned into a finite number of pairwise disjoint sets $P_0,\ldots, P_k$,
in such a way that, for each rule whose head is in $P_i$, each predicate occurring in a positive literal in the body must belong to some $P_j$ with $j\leq i$,
and each predicate occurring in a negative literal in its body must belongs to some $P_j$ with $j<i$.
Indeed, three pairwise disjoint sets of predicates $P_0$, $P_1$ and $P_2$ can be defined , where $P_0$ contains all the predicates in  $\Pi^{Pos}_{RT}(K)$,
$P_1$ contains predicates $\mathit{rank}$ and $\mathit{newNonEx}$, and $P_2$ contains predicates $\mathit{fixp}$, $\mathit{inf\_Rank}$, $\mathit{inrc1}$ and $\mathit{inrc2}$.

The non-disjunctive, stratified Datalog program $\Pi_{RT}(K)$ has a unique perfect model,  coinciding with the unique stable model of the program \cite{Przymusinski91}. When restricted to the predicates in $P_0$, occurring in the positive part $\Pi^{Pos}_{RT}(K)$, the unique stable model of $\Pi_{RT}(K)$ coincides with the minimal model of $\Pi^{Pos}_{RT}(K)$ (i.e. with the set of facts derivable from $\Pi^{Pos}_{RT}(K)$ in Datalog), as  a consequence of a general property of modularized disjunctive Datalog programs in \cite{Eiter97} (Lemma 5.1).
The fact that a concept $C$ has rank $i$ ($\mathit{rank(C,I)}$), as well as the existence of a concept with rank $i$ ($\mathit{newNonEx(I)}$), and the fact that the iteration has reached a fixedpoint at $i$ ($\mathit{fixp(I) }$) can then all be computed stratum by stratum according to the rules (C7)-(C12). Their correctness is evident from their definition, and we omit a formal proof.
For $K$ in Example \ref{exa:sec5}, the program $\Pi_{RT}(K)$ has a unique stable model in which $\mathit{rank(aux_\top,0), rank(aux_A,0), rank(aux_B,0),}$ $\mathit{rank(aux_D,1), rank(a,1)}$ hold.

%
%

Observe that computing the ranks of all concepts $C$ such that $\tip(C)$ occurs in the KB (or in the query) requires a quadratic number
of exceptionality checks in the size of the KB (as the maximum rank value that can be assigned to a concept is bounded by the  number of typicality inclusions in the KB).
Each exceptionality check requires polynomial time, using the above mentioned 
variant of polynomial calculus for subsumption, by a call to predicate ${inst\_h(C,Z,C,I)}$.
As observed, the computation of the ranking can be done stratum by stratum and
requires polynomial time.

The correspondence between the rational closure construction and the canonical minimal model semantics in  \cite{AIJ15}
does not extend to all the constructs in ${\cal SROEL} (\sqcap,$ $\times)^{\Ra}\tip$ and, specifically, the canonical model semantics is not adequate 
for dealing with nominals. 
In particular, there are knowledge bases with no canonical model and knowledge bases with more than one minimal canonical model
(as the knowledge base in Example \ref{exa: estensioni multiple}). 
However, in many cases, the rational closure of a KB with no canonical model is still meaningful.
As an example, observe that the TBox consisting of the rules (a)-(f) in Example \ref{exa:1} is simple, but has no canonical model (in fact, although the concepts
$\mathit{\{black\} \sqcap MathHather}$ and $\mathit{\{black\} \sqcap MathLover}$ are both consistent with $K$, there is no model of $K$ in which they  both have an instance). However, the rational closure of this TBox is consistent, and entails, for instance, that all the typical young Italians have black hair. In particular, the concepts {\em student} and {\em young Italian} are given rank $0$, while {\em nerd student} has rank $1$ and {\em nerd student and math hater} has rank $2$.
A less restrictive semantic requirement to give meaning also to KBs containing nominals has been considered in \cite{TPLP2016},
where $\tip$-minimal models are introduced, corresponding to the minimal models among the ones which contain at least one instance of any consistent concept occurring within the typicality operator in the KB. 
We refer to \cite{TPLP2016} for a detailed description of this semantics and of its relation with the minimal canonical model one. 
The two semantics coincide when minimal canonical models of the KB exist and give to all the concepts occurring in the $\tip$ operator  the same rank  computed by the rational closure construction. 
One can expect that the correspondence among the rational closure and the $\tip$-minimal model semantics extends to a larger fragment of the language including the TBox (a)-(f) above, for which a unique $\tip$-minimal model exists. While in this paper we do not address the issue of determining such a fragment and establishing the correspondence, we notice that checking the consistency of the rational closure provides a simple way to identify the KB for which the rational closure construction is meaningful.
The consistency of the rational closure can be easily determined using the materialization calculus.
The following rules:
\begin{tabbing}
	$(\mathit{SameRank\_rc1})  ~ \mathit{sameRank(A_C,A_D) \leftarrow auxtc(A_C,C),auxtc(A_D,D),}$\\
	\hspace{6cm} \ \  $\mathit{ rank(C,I), rank(D,I). } $ \\
	$(\mathit{LeqRank\_rc2})  ~ \mathit{leqRank(A_C,A_D) \leftarrow auxtc(A_C,C),auxtc(A_D,D),}$\\
	\hspace{5cm} \ \ \ \ \ $\mathit{ rank(C,I),rank(D,J), I<J.} $
\end{tabbing}
are added to $\Pi_{RC}$ to incorporate the information on the ranks computed by the rational closure construction in the materialization calculus.
If the resulting program derives $\mathit{inst(x,A)}$ for any concept $A$ such that $bot(A) \in \Pi_K$, then the rational closure of $K$ is inconsistent,
as  the rank assignment computed by the rational closure does not correspond to any model of the KB
(which can be proved, by a simple generalization of the soundness proof in Theorem \ref{ASP}).

Note that, if the inclusion  $\exists U.(\{a\} \sqcap \tip(\top)) \sqcap\exists U.(\{b\} \sqcap \tip(\top)) \sqsubseteq \bot$ (similar to the inclusion
in the multiple extension Example  \ref {exa: estensioni multiple})
were added to the KB from Example \ref{exa:1}, with the TBox consisting of rules (a)-(f) as above, the computed ranks would not change w.r.t. those given above. However, if $\tip(\{a\}) \sqsubseteq C$ and $\tip(\{b\}) \sqsubseteq D$ were added, 
the rational closure would become inconsistent
as it would assign to both concepts $\{a\}$ and $\{b\}$ a rank, namely $0$, but no model of the KB exists which such ranks.

For a consistent KBs in $\sroelrt$ whose rational closure is inconsistent, two options are available: either to reason (skeptically) on the alternative $\tip$-minimal models of the KB (if any) by exploiting the ASP encoding of the $\tip$-minimal model semantics in \cite{TPLP2016}, or to take the inconsistency as a clue that there are potentially unresolved conflicts in the KB, concerning the inheritance of alternative defeasible inclusions from more general to more specific classes, to be resolved by modifying the KB (an approach adopted for the logic of overriding in \cite{bonattiAIJ15}).

\section{Related Work}

Among the recent nonmonotonic extensions of DLs  are the formalisms for combining DLs with logic programming rules,
such as for instance, \cite{Eiter2008,Eiter2011,rosatiacm,KnorrECAI12} and Datalog +/- \cite{Gottlob14}.
DL-programs \cite{Eiter2008,Eiter2011} support a loose coupling of DL ontologies and rule-based reasoning
under the answer set semantics and the well-founded semantics; rules may contain DL-atoms in their bodies,
corresponding to queries to a DL ontology, which can be modified according to a list of updates.
In \cite{KnorrECAI12} a general DL language is introduced,
which extends ${\cal SROIQ}$ with nominal schemas and epistemic operators according to the MKNF semantics \cite{rosatiacm},
which encompasses prominent nonmonotonic rule languages, including ASP.
In \cite{bonattiAIJ15} a non monotonic extension of DLs is proposed based on a notion of overriding, supporting normality concepts and
enjoying good computational properties, preserving the tractability of low complexity DLs, including ${\el}^{++}$ and $DL$-$lite$.
In \cite{Bozzato14} the CKR framework  is presented; it is
based on {\em SROIQ-RL},
allows for defeasible axioms with local exceptions and exploits a translation to Datalog with negation.
It is shown that instance checking in CKR reduces to (cautious) inference under the answer set semantics.

Preferential extensions of low complexity DLs in the $\cal{EL}$ and DL-lite families have been studied 
in \cite{GiordanoLPNMR09,ijcai2011},
based on preferential interpretations which are not required to be modular,
and tableaux-based proof methods have been developed for them.
In \cite{ijcai2011}, for a preferential extension of $\elbot$
based on a minimal model semantics different from the one in this paper,
it is shown that minimal entailment is \textsc{ExpTime}-hard already for simple KBs,
similarly to what happens for circumscriptive KBs  \cite{Bonatti2011}.

The first notion of rational closure for DLs was defined by Casini and Straccia \cite{casinistraccia2010};
their rational closure construction for $\alc$  directly uses entailment in $\alc$ over a materialization of the KB.
A variant of this notion of rational closure has been studied in \cite{CasiniDL2013}. 
To overcome the limitations of rational closure,
in \cite{CasiniJAIR2013} an approach is introduced based on the combination of rational closure
and \emph{Defeasible Inheritance Networks},  while in  \cite{Casinistraccia2012}  the lexicographic closure introduced by Lehmann \cite{Lehmann95} is extended to description logics.
Furthermore, in \cite{fernandez-gil} an extension of $\alct$ with several typicality operators is proposed, each corresponding to a preference relation,
and in  \cite{GliozziNMR2016}  
a refinement of the rational closure is developed,
where  models are equipped with several preference relations.
Whether the presented approach can be extended to refinements of rational closure will be explored in future work.

In \cite{TPLP2016} a rule based inference method for $\sroelrt$ minimal entailment based on model generation in ASP has been developed;
here we exploit Datalog plus stratified negation to construct the rational closure of a KB. As discussed above, the two approaches are complementary. 
Related approaches are the work 
in \cite{Isberner2016} which characterizes skeptical c-inference as a constraint satisfaction problem,
and the work  
in  \cite{Russo2015}, who presents an inconsistency tolerant semantics for the Description Logic using preference weights and exploit ASP optimization for computing preferred interpretations.

\section{Conclusions}

In this paper we have studied a rational extension $\sroel^{\Ra \ }\tip$ of the
low complexity description logic $\sroel$, which underlies the OWL EL ontology language,
introducing a typicality operator.
For general KBs, 
we have shown that minimal entailment in $\sroelrt$ is \textsc{$\Pi^P_2$}-hard.
When free occurrences of typicality concepts in concept inclusions are allowed,
alternative minimal models may exist with different
rank assignments to concepts.
In \cite{Booth15} this phenomenon has been analyzed in the context of PTL,
considering alternative preference relations over ranked interpretations 
which coincide over simple KBs but, for general ones, define
different notions of entailment satisfying alternative and possibly incompatible postulates.

Building on the calculus for $\sroel$ in Datalog presented in \cite{KrotzschJelia2010},
a calculus for instance checking and subsumption under rational entailment is defined,
showing that these problems can be decided in polynomial time.
A preliminary  version of this result appeared in \cite{DL2016,CILC2016}.
This calculus is extended to provide a polynomial construction of the rational closure of a knowledge base in $\sroelrt$, using Datalog with stratified negation  \cite{Przymusinski91}.
Although the minimal canonical model semantics provides a characterization for rational closure of simple $\sroelrt$ knowledge bases 
on the fragment of the language only including the constructs in $\alc$ (by a result in \cite{AIJ15}), a more general semantic characterization for a wider fragment of the language is still to be developed. In this respect, a promising semantics is the $\tip$-minimal model semantics proposed in \cite{TPLP2016}.

Future work may also include optimizations, based on modularity as in \cite{BonattiSWC15}, of the calculus for rational entailment, 
the development of a multi-preference construction to address the drawbacks of rational closure
as the development of Abox minimization techniques.
A further issue to understand is whether a calculus can be defined also for the preferential extensions of DLs in the $\cal{EL}$ family
studied in \cite{GiordanoLPNMR09,ijcai2011}, 
whose interpretations are not required to be ranked, as well as for the logics in the DL-lite family, 
for which inconsistency tolerant semantics have been developed \cite{Lembo2015,Papini2016}.

Apart from providing a polynomial complexity upper bound, the encoding presented in this paper
is intended to provide a way to integrate the use of $\sroel$ KBs under rational entailment 
with other kinds of reasoning that can be performed in ASP,
and, by modifying the encoding, 
also to allow the experimentation of alternative notions of minimal entailment, as advocated in \cite{Booth15}.
The approach can be integrated with systems like DReW \cite{Xiao2012},
that already exploits the mapping by Kr\"{o}tzsch for OWL 2 EL.

\medskip
{\bf Acknowledgement}. We thank the anonymous referees for their helpful comments
This research has been supported by INDAM - GNCS Project 2016 {\em Ragionamento Defeasible nelle Logiche Descrittive}.

\end{document}